\documentclass[11pt]{article}

\usepackage{fullpage,times}

\pdfoutput=1

\usepackage{times}

\usepackage[utf8]{inputenc}
\usepackage{hyperref}
\hypersetup{
        colorlinks   = true,
        urlcolor     = blue,
        linkcolor    = blue,
        citecolor   = blue
      }

\usepackage{bm}
\usepackage{url}
\usepackage{booktabs}
\usepackage{amsmath,amsfonts,amssymb,thmtools,amsthm}
\usepackage{algorithm2e}
\usepackage[capitalize]{cleveref}
\usepackage{nicefrac}
\usepackage{microtype}
\usepackage{natbib}
\usepackage{setspace}
\usepackage{yhmath}

\usepackage{enumitem}

\usepackage[nolist,nohyperlinks]{acronym}
\usepackage{color}
\usepackage{bold-extra}
\usepackage{mathtools}
\usepackage{etoolbox}
\usepackage{xfrac}
\usepackage{nomencl}
\usepackage{graphicx}
\usepackage{wrapfig}
\usepackage{float}
\usepackage{constants}

\RestyleAlgo{ruled}

\newcommand{\pr}[1]{\left( #1 \right)}

\newcommand{\cbr}[1]{\left\{ #1 \right\}}
\newcommand{\abs}[1]{\left|#1\right|}
\newcommand{\ceil}[1]{\left\lceil #1 \right\rceil}

\newcommand{\lf}{\left}
\newcommand{\rt}{\right}
\newcommand{\bmid}{\;\middle|\;}
\newcommand{\tp}{^{\intercal}}
\newcommand{\ip}[1]{\left\langle #1 \right\rangle}
\newcommand{\ind}[1]{\mathbb{I}\cbr{ #1 }}

\renewcommand{\P}{\operatorname{\mathbb{P}}} %
\newcommand{\E}{\operatorname{\mathbb{E}}} %

\newcommand*\diff{\mathop{}\!\mathrm{d}}

\DeclareMathOperator*{\argmin}{arg\,min}

\newcommand{\lmin}{\lambda_{\min}}
\newcommand{\lmax}{\lambda_{\max}}

\newcommand{\Xdom}{\mathbb{S}^{d-1}}

\newcommand{\df}{\stackrel{\mathrm{def}}{=}}
\newcommand{\leqC}{\lesssim}
\newcommand{\geqC}{\gtrsim}

\newcommand{\bM}{\mathbf{M}}

\newcommand{\bx}{\mathbf{x}}

\newcommand{\by}{\mathbf{y}}

\newcommand{\btilx}{\mathbf{\tilde{x}}}
\newcommand{\btily}{\mathbf{\tilde{y}}}

\newcommand{\bu}{\mathbf{u}}

\newcommand{\bw}{\mathbf{w}}

\newcommand{\bh}{\mathbf{h}}
\newcommand{\bg}{\mathbf{g}}

\newcommand{\bv}{\mathbf{v}}

\newcommand{\btilw}{\mathbf{\tilde{w}}}

\newcommand{\bK}{\mathbf{K}}
\newcommand{\bhK}{\mathbf{\hat{K}}}
\newcommand{\bI}{\mathbf{I}}

 \newcommand{\btheta}{\boldsymbol{\theta}}
\newcommand{\bbartheta}{\boldsymbol{\bar{\theta}}}

\newcommand{\btiltheta}{\boldsymbol{\tilde{\theta}}}
\newcommand{\bhtheta}{\boldsymbol{\hat{\theta}}}
\newcommand{\hL}{\hat{L}}

\newcommand{\sC}{\mathcal{C}}

\newcommand{\sE}{\mathcal{E}}
\newcommand{\sF}{\mathcal{F}}

\newcommand{\sH}{\mathcal{H}}

\newcommand{\sR}{\mathcal{R}}

\newcommand{\sN}{\mathcal{N}}

\newcommand{\sO}{\mathcal{O}}
\newcommand{\stilO}{\tilde{\mathcal{O}}}
\newcommand{\bPhi}{\mathbf{\Phi}}

\newcommand{\bphi}{\boldsymbol{\phi}}
\newcommand{\balpha}{\boldsymbol{\alpha}}

\newcommand{\bbeta}{\boldsymbol{\beta}}
\newcommand{\bgamma}{\boldsymbol{\gamma}}
\newcommand{\bbaralpha}{\boldsymbol{\bar{\alpha}}}

\newcommand{\bzero}{\boldsymbol{0}}

\newcommand{\bff}{\boldsymbol{f}} %

\newcommand{\ve}{\varepsilon}
\newcommand{\bve}{\boldsymbol{\varepsilon}}
\newcommand{\Uset}{\cbr{\pm 1/\sqrt{m}}^m}

\newcommand{\reals}{\mathbb{R}}

\newcommand{\fstar}{f^{\star}}

\newcommand{\superlabel}[1]{^{\mathrm{#1}}}

\newcommand{\rf}{\superlabel{rf}}
\newcommand{\ntk}{\superlabel{\kappa}}

{%
  \endtrivlist
}%

{\hfill$\square$}

\definecolor{light-gray}{gray}{0.5}

\newcommand{\poly}{\mathrm{poly}}
\newcommand{\unif}{\mathrm{unif}}

\newcommand{\hT}{\widehat{T}}

\newcommand{\conf}{\nu}

\newcommand{\polylog}{\mathrm{polylog}}
\newcommand{\kexcess}{\epsilon(n, \hT)}
\newcommand{\khexcess}{\widehat{\epsilon}(n, \hT)}
\newcommand{\kexcessfun}{\epsilon}
\newcommand{\khexcessfun}{\widehat{\epsilon}}

\makeatletter
\newcommand\x@overcf[2]{\overset{\vphantom{x}\smash{\raisebox{-0.5ex}{$#1\frown~$}}}{#2}}
\newcommand\overcf[1]{\mathpalette\x@overcf{#1}}
\makeatother %
\begin{acronym}
\acro{RLS}{Regularized Least Squares}
\acro{ERM}{Empirical Risk Minimization}
\acro{RKHS}{reproducing kernel Hilbert space}
\acro{DA}{Domain Adaptation}
\acro{PSD}{Positive Semi-Definite}
\acro{SGD}{Stochastic Gradient Descent}
\acro{OGD}{Online Gradient Descent}
\acro{GD}{Gradient Descent}
\acro{SGLD}{Stochastic Gradient Langevin Dynamics}
\acro{IS}{Importance Sampling}
\acro{WIS}{Weighted Importance Sampling}
\acro{MGF}{Moment-Generating Function}
\acro{ES}{Efron-Stein}
\acro{ESS}{Effective Sample Size}
\acro{KL}{Kullback-Liebler}
\acro{SVD}{Singular Value Decomposition}
\acro{PL}{Polyak-Łojasiewicz}
\acro{NTK}{Neural Tangent Kernel}
\acro{KLS}{Kernel Least-Squares}
\acro{KRLS}{Kernelized Regularized Least-Squares}
\acro{ReLU}{Rectified Linear Unit}
\acro{NTRF}{Neural Tangent Random Feature}
\acro{NTF}{Neural Tangent Feature}
\acro{RF}{Random Feature}
\acro{RWY}{Raskutti-Wainwright-Yu}
\acro{CW}{Celisse-Wahl}
\end{acronym} 
\newcommand{\jmlrQED}{\hfill$\square$}
\renewcommand{\Cref}[1]{\cref{#1}}

\newtheorem{theorem}{Theorem}
\newtheorem{definition}{Definition}
\newtheorem{lemma}{Lemma}
\newtheorem{corollary}{Corollary}
\newtheorem{remark}{Remark}
\newtheorem{proposition}{Proposition}
\newtheorem{assumption}{Assumption}

\renewconstantfamily{normal}{
  symbol=c,
  format=\arabic
}

\newconstantfamily{log}{
  symbol=\tilde{c},
  format=\arabic}
\newconstantfamily{eps}{
  symbol=\epsilon,
  format=\arabic}

\newcommand{\nop}[1]{}
\newconstantfamily{init}{
  symbol=c_{\mathrm{init}},
  format=\nop}
\newconstantfamily{width}{
  symbol=c_{\mathrm{width}},
  format=\nop}

\title{Learning Lipschitz Functions by GD-trained Shallow Overparameterized ReLU Neural Networks}
\date{}

\author{%
  Ilja Kuzborskij\\
  DeepMind, London\\
  \texttt{iljak@deepmind.com}
  \and
  Csaba Szepesv\'ari\\
  DeepMind, Canada and University of Alberta, Edmonton\\
  \texttt{szepi@deepmind.com}
}

\begin{document}

\maketitle

\begin{abstract}%
  We explore the ability of overparameterized shallow ReLU neural networks to learn Lipschitz, non-differentiable, bounded functions with additive noise when trained by \ac{GD}.
  To avoid the problem that in the presence of noise,
  neural networks trained to nearly zero training error are inconsistent in this class,
  we focus on the early-stopped \ac{GD} which allows us to show consistency and optimal rates.
  In particular, we explore this problem from the viewpoint of the \ac{NTK} approximation of a \ac{GD}-trained finite-width neural network.
  We show that whenever some early stopping rule is guaranteed to give an optimal rate (of excess risk) on the Hilbert space of the kernel induced by the ReLU activation function, the same rule can be used to achieve
  minimax optimal rate for learning on the class of considered Lipschitz functions by neural networks.
  We discuss several data-free and data-dependent practically appealing stopping rules that yield optimal rates.
\end{abstract}

\section{Introduction}
\label{sec:setting}
In this work we consider the setting where the learner is given a \emph{training sample} $S = (\bx_i, y_i)_{i=1}^n$
consisting of \emph{inputs} $\bx_i$ and \emph{targets} $y_i$.
Each input $\bx_i$ is independently sampled from some unknown probability measure $P_X$ over $\Xdom = \{\bx \in \reals^{d} ~:~ \|\bx\|_2 = 1\}$, and targets are generated independently from each other according to the regression model
\[
  y_i = \fstar(\bx_i) + \ve_i \qquad (i \in [n])
\]
where $\ve_1, \ldots, \ve_n$ are independent bounded noise variables satisfying $\E[\ve_i \mid \bx_i] = 0$ and $\E[\ve_i^2 \mid \bx_i] = \sigma^2$.
Here function $\fstar$ is called the \emph{regression function}, and we assume that it is uniformly bounded and Lipschitz, meaning that it satisfies $|\fstar(\bx) - \fstar(\btilx)| \leq \Lambda \|\bx - \btilx\|$ for all $\bx, \btilx \in \Xdom$ for some $0 \leq \Lambda < \infty$.
Note that we do not assume that $\fstar$ is differentiable.
The above implies that the targets are bounded, and throughout the paper we have $|y_i| \leq B_y$ for some finite constant $B_y$.
In the following we consider learning $\fstar$ by the \emph{shallow neural network predictor} defined as
\begin{align*}
  f_{\btheta}(\bx) = \sum_{k=1}^m u_k (\bw_k\tp \bx)_+ \qquad \pr{\btheta = (\bw_1, \ldots, \bw_m) \in \reals^{d m}, \,\, \bx \in \Xdom}~.
\end{align*}
Here $(x)_+ = \max\cbr{x,0}$ is called the \ac{ReLU} \emph{activation function},
$m$ is the \emph{width} of the network, $\btheta$ is a tunable parameter vector of the \emph{hidden layer}, and $\bu \in \Uset$ is a fixed parameter vector of the output layer (\Cref{alg:GD} describes parameter setting).
In particular, we are interested in the ability of such a shallow neural network fitted on the (noisy) sample $S$ to learn the regression function.
This ability is captured by the \emph{excess risk}
\begin{align*}
  \|f_{\btheta} - \fstar\|_2^2 = \int_{\Xdom} (f_{\btheta}(\bx) - \fstar(\bx))^2 \diff P_X(\bx)~.
\end{align*}
The goal of a learning procedure is to minimize the excess risk as quickly as possible based on the sample.
At the very least we require that the procedure producing parameters $\bhtheta_n$ is asymptotically \emph{consistent}, meaning that $\|f_{\bhtheta_n} - \fstar\|_2^2 \to 0$ as $n \to \infty$ in probability.
It is well known that in the non-asymptotic sense, for the considered class of regression functions, the minimax excess risk behaves as $\|f_{\bhtheta_n} - \fstar\|_2^2 = \Theta_{\P}(n^{-\frac{2}{2+d}})$~\citep{gyorfi2006distribution}.

In this paper we are interested in a learning scenario where to achieve this goal,
parameters  are tuned by minimization of the \emph{empirical risk}
\begin{align*}
  \hL(\btheta) = \frac1n \sum_{i=1}^n (f_{\btheta}(\bx_i) - y_i)^2~.
\end{align*}%
Note that $\hL(\cdot)$ is a non-convex function due to non-linearity of the \ac{ReLU} function $(\cdot)_+$.
In particular, for minimization of $\hL(\cdot)$ we employ \ac{GD} procedure given in \Cref{alg:GD}.
\ac{GD} has some of the parameters on its own, that is the number of steps $T$, step size $\eta$, and the width of the neural network $m$.\footnote{Without loss of generality, $m$ is an even integer.}
Results presented later on, consider a fully data-dependent tuning of $(\eta, T, m)$.
\begin{algorithm}  
  \caption{\acl{GD} for training shallow neural networks.}\label{alg:GD}
    \KwData{
      $S$ -- training sample, $T$ -- number of steps, $\eta \in (0, \frac12]$ -- step size, $m$ -- network width.
      }
    \KwResult{Trained parameters of a hidden layer $\btheta_T$.}
    Set $\bu = (u_1, \ldots, u_m)$ as $u_k = -\frac{1}{\sqrt{m}}$ for $k \leq \frac{m}{2}$ and $u_k = \frac{1}{\sqrt{m}}$ for $k > \frac{m}{2}$\;
    Set $\btheta_0 = (\bw_{0,1}, \ldots, \bw_{0,m})$ as $\bw_{0,k} \stackrel{\mathrm{i.i.d.}}{\sim} \sN(\bzero, \bI_d)$ for $k \leq \frac{m}{2}$, and $\bw_{0,\frac{m}{2}+k} = \bw_{0,k}$ for $k > \frac{m}{2}$\;
    \For{$t=1,\ldots,T$}{
    $\btheta_{t} \gets \btheta_{t-1} - \eta \nabla \hL(\btheta_{t-1})$\;
  }
  \end{algorithm}
\begin{remark}
  \label{rem:initial}
  The choice of initial parameters $(\bu, \btheta_0)$ in \Cref{alg:GD} ensures that $f_{\btheta_0}(\bx) = 0$ for all $\bx$ and moreover that $\hL(\btheta_0) \leq B_y^2$.
\end{remark}
\paragraph{Prior work.}
Understanding learning ability of algorithms for training neural networks often concerns behavior of the \emph{statistical risk} (a.k.a.\ the risk) %
defined as
$L(\btheta) = \E[(f_{\btheta}(\bx) - y)^2 \mid \btheta]$,
which is closely related to the excess risk studied here,
\begin{align*}
  \E[\|f_{\btheta} - \fstar\|_2^2 \mid \btheta] = L(\btheta) - \sigma^2~.
\end{align*}
Study of the risk in neural network learning is a long-lived topic of interest~\citep{devroye1996probabilistic,anthony1999neural}.
The classical approach to this problem considers \emph{uniform bounds} on the risk, which hold simultaneously over all members of a given class of neural networks for all data distributions.
In such an algorithm-free analysis, bounds are controlled by some notion of the capacity of the class, such as VC-dimension, Rademacher complexity, or metric entropy~\citep{anthony1999neural,bartlett2002rademacher,bartlett2017spectrally,golowich2018size}.
Clearly, these bounds due to their generality can be used to argue about the risk of the predictor generated by a concrete algorithm, such as \ac{GD} we consider here.
To this end, this gives us high-probability bounds on the risk which capture the trade-off between how well the algorithm fits the sample and complexity of the solution (which is measured by some norm of the parameters, such as $\ell^2$):
\begin{align*}
  L(\btheta_T) \leq \hL(\btheta_T) + c \, \frac{\text{network-complexity}(\btheta_T)}{\sqrt{n}}~.
\end{align*}
Thus, if one could minimize the empirical risk while keeping network complexity under control, one could keep the risk close to the noise rate $\sigma^2$.
Unfortunately, minimization of such an objective is theoretically challenging because of its non-convex nature, and so it is non-obvious whether \ac{GD} can minimize it till the required precision in general.
On the other hand, nowadays it is a common knowledge that \emph{overparameterized} neural networks (when the width $m$ by far exceeds the sample size $n$) are able to achieve arbitrarily small training error on any training sample.
From a statistical perspective, driving $\hL(\btheta_T) \to 0$ should lead to fitting the noise, or \emph{overfitting}, however \citet{zhang2021understanding} and others concluded that training neural networks by \ac{GD}-type algorithms to near-zero empirical risk results only in a mild overfitting.
As such, these algorithms are often able to perform well on the unseen sample despite of the absence of explicit complexity control.

To this end, from the optimization viewpoint, some tangible progress was recently made by \cite{du2018gradient,allen2019convergence,oymak2020towards,zou2020gradient,nguyen2021tight,song2021subquadratic} and others, who proposed an explanation why
\ac{GD} is able to minimize the empirical risk in the overparameterized regime.
In particular, these works show an exponential convergence rate with high probability:\footnote{Throughout this paper, we use $f \leqC g$ to say that there exists a universal constant $c > 0$ such that $f \leq c \, \mathrm{polylog}(g) \, g$ holds uniformly over all arguments.}
\begin{align*}
  \hL(\btheta_T) \leqC \pr{1-c \eta \, \frac{d}{n}}^T~.
\end{align*}
Roughly speaking, these proofs are based on the observation that under sufficient overparameterization ($m = \poly(n,d)$), the predictor $f_{\btheta_T}$ behaves similarly to the linear predictor defined on a large feature space constructed from the gradient of the network around its (randomized) initialization, a so-called \ac{NTF} space.
At the same time, considering the setting when $m \to \infty$, one can show~\citep{jacot2018neural} that under appropriate conditions such a feature space gains approximation power of a \ac{RKHS} \citep{cucker2007learning}, giving rise to the \ac{NTK} theory.
To some extent, this connection should not come as a surprise since \ac{NTF} is a particular case of the \ac{RF} framework well-studied over decades~\citep{rahimi2007random,rudi2017generalization}.

The \ac{NTK} theory established a bridge between neural network learning and the \ac{RKHS} literature, and as such led to many interesting results (see review by \cite{bartlett2021deep}).
Going back to the question of the risk behavior,
\ac{NTK} theory might also be used as a tool to understand complexity of a resulting predictor by looking at it through the lens of learning with kernels.
A number of works put forward possible explanations for the complexity control in \ac{GD}-trained neural networks suggesting that the \ac{GD} algorithm is responsible for an \emph{implicit} complexity control (or implicit bias), and analyzed it through the \ac{RKHS} viewpoint~\citep{arora2019fine,allen2019learning,ji2019polylogarithmic,hu2021regularization}.
For example,~\citet{arora2019fine} showed that in the \emph{noise-free} setting ($\sigma = 0$) the risk is controlled by the norm of a minimal-norm $\ell^2$ kernel interpolant (predictor perfectly fitting the sample),
\begin{align*}
  \lim_{T \to \infty} L(\btheta_T) = \sO_{\P}\pr{\frac{\|\text{minimal-norm-interpolant}\|_2}{\sqrt{n}}} \quad \text{as} \quad n \to \infty~.
\end{align*}
Moreover, \cite{arora2019fine} demonstrated several examples when the norm above can be bounded when the regression function $\fstar$ is a shallow polynomial neural network of infinite width.
Results in a similar spirit were also obtained for a hard-margin classification setting~\citep{ji2019polylogarithmic}.

At the same time, for $\sigma > 0$, it is known that \emph{interpolating} neural networks (achieving zero empirical risk) are in general \emph{inconsistent}~\citep{kohler2019over,hu2021regularization}.
To this end, Corollary~1 of~\cite{kohler2019over} implies that as long as $\hat L(\btheta_T)=o(1/n)$, for any large enough $n$
there exists a distribution $P$ with, say, $\sigma^2=1/4$, such that $\E L(\btheta_T)-\sigma^2 \ge c$ for some universal constant $c>0$, where the distribution can even be chosen to be `sufficiently regular', though the marginal $P_X$ with respect to the inputs will be an atomic distribution.

\subsection{Our contribution}
In this work we focus on the case $\sigma > 0$ and `complex' regression functions $\fstar$, that is, non-differentiable, Lipschitz, and bounded.
The setting of learning such functions is called \emph{nonparametric}, where in general the number of parameters (here, width $m$), has to increase with the number of observations for us to be able to estimate $\fstar$ well.
Recall that the minimax-optimal rate of the excess risk in this setting is $\Theta_{\P}(n^{-\frac{2}{2+d}})$~\citep{gyorfi2006distribution}.
In particular, we show that early-stopped \ac{GD} without explicit penalization (\Cref{alg:GD}) used for training overparameterized shallow ReLU neural networks is \emph{consistent} in the considered setting
and \emph{matches} this optimal rate.

Our main result is a `reduction' theorem, which allows us to convert any excess risk bound for the early-stopped \ac{GD} learning in the \ac{RKHS} $\sH$ (induced by the \ac{NTK} of ReLU), to the excess risk bound for \ac{GD} using the same stopping rule and learning Lipschitz functions by neural networks.
More formally, suppose that given any function $h \in \sH$
with $\|h\|_{\sH}^2 \leq R$,\footnote{$\|\cdot\|_{\sH}$ is $\ell^2$ norm on RKHS --- see \Cref{sec:rkhs} for a summary of RKHS definitions and notation.} we can guarantee that \ac{GD} stopped according some rule at step $\hT$ achieves excess risk bound $(1+R) \, \kexcess$.
Then, in \Cref{thm:random-design} we show that a \ac{GD}-trained neural network $f_{\hT}$ of width $m \geq \poly_6\pr{(1+1/d) n}$, with high probability satisfies\footnote{Here $\poly_k(\cdot)$ stands for a polynomial of degree $k$.}
\begin{align}
  \label{eq:pre-rate}
  \|f_{\hT} - \fstar\|_2^2 = \stilO_{\P}\pr{\Lambda^2 (\kexcess)^{\frac2d} + \kexcess + \frac{\poly_4( n/d)}{\sqrt{m}}} \qquad \text{as} \qquad n \to \infty~.
\end{align}
Here, term $\Lambda^2 (\kexcess)^{\frac2d}$ is a price we pay for approximation of a Lipschitz function $\fstar$ by a function in \ac{RKHS},
term $\kexcess$ is a cost of learning such a function in \ac{RKHS},
and term $\poly_4(\cdot) / \sqrt{m}$ is a gap we pay for prediction with a neural network instead of learning on \ac{RKHS} directly --- more overparameterization means smaller gap
(see \Cref{sec:sketch} for the sketch of the analysis).

For now suppose that we chose an early-stopping rule for \ac{GD} which guarantees a minimax optimal rate of excess risk for learning on $\sH$, that is $\kexcess \leqC n^{-\frac{d}{2+d}}$ (we will get back shortly to some concrete rules).
Then, it is clear that for a sufficiently (polynomially) overparameterized neural network and for $d \geq 2$, we achieve an optimal nonparametric rate
\begin{align}
  \label{eq:rate}
  \|f_{\hT} - \fstar\|_2^2 = \stilO_{\P}\pr{\Lambda^2 \, n^{-\frac{2}{2+d}}} \qquad \text{as} \qquad n \to \infty~.
\end{align}
Polynomial overparameterization considered here is a standard assumption in the literature on training of shallow networks by \ac{GD}~\citep{du2018gradient,oymak2020towards}.
The rate of \cref{eq:rate} is optimal in the sample size and input dimension, however it is suboptimal in the Lipschitz constant $\Lambda$ as one would expect to see $\Lambda^{\frac{2}{2+d}}$ in place of $\Lambda^2$.
This worse dependence comes at the expense of an agnostic algorithm, which does not require any knowledge of $\Lambda$.
A similar bound also holds for the \emph{fixed design} scenario, namely where inputs $\bx_1, \ldots, \bx_n$ are fixed, with a polynomial term improved till $\poly_3(\cdot)$, see \Cref{thm:fixed-design}.

In contrast to \cref{eq:rate}, a minimax optimal rate of learning on \ac{RKHS} balls is $\Theta(n^{-\frac{d}{d+2}})$ \citep{steinwart2009optimal,bach2017breaking} --- a much faster rate, however
such balls only contain \emph{highly smooth functions}:
\citet{bietti2019inductive} identified that these are \emph{at least} $\ceil{d/2}$ times differentiable even functions on a sphere with bounded derivatives. %
Thus, the rate shown here suggest that \ac{GD}-trained neural networks can potentially learn much more complex functions than those within fixed \ac{RKHS} balls.
On the other hand,
\citet{bach2017breaking} observed that such \ac{RKHS}es seems to be rich enough to approximate non-differentiable Lipschitz functions
and that the approximation error decreases as we consider approximators of increasing norm.
In this work we exploit this fact together with observation that an early-stopped \ac{GD} is capable of learning such approximators at an optimal rate.

\paragraph{Non-empirical stopping rules.}
Now we turn our attention to a concrete example of a stopping rule which yields optimal rate of \cref{eq:rate}.
In particular, we consider the rule of \citet{dieuleveut2016nonparametric}, which
requires some knowledge about the regularity of the \ac{RKHS}.
One such standard regularity condition is the eigenvalue decay rate of the kernel function (given by the Mercer's theorem, see \Cref{sec:rkhs}), which roughly speaking captures how smooth functions belonging to the \ac{RKHS} can be.
In our case we consider a polynomial decay rate of order $k^{-\beta}$ for some $\beta > 0$: \ac{RKHS} of the \ac{NTK} derived from the \ac{ReLU} activation is known to have decay rate of order $k^{-\frac{d}{2}}$ \citep{bach2017breaking}. %
\citep[Corollary 3]{dieuleveut2016nonparametric} showed that by using \ac{GD} the step size $\eta = \tfrac12 n^{-\frac{1}{1+\beta}}$ and $\hT = n$, \ac{GD} enjoys excess risk of order
\begin{align*}
  \E[\kexcess] \leqC n^{-\frac{\beta}{\beta+1}}~.
\end{align*}
Combining this result (for $\beta = d/2$) with \cref{eq:pre-rate} (by using Markov's inequality) we indeed recover optimal rate of \cref{eq:rate} with the loss of a high probability guarantee.

It is interesting to ask which decay rate of eigenvalues is required to get optimal rates~\citep{pillaud2018statistical}.
For instance, we can consider a classical analysis of early-stopped \ac{GD} of~\cite{yao2007early} which states that for the number of steps chosen as $\hT = n^{\frac13}$, with high probability we have excess risk of order $n^{-\frac{1}{3}}$ for learning in \ac{RKHS}.
This results in a suboptimal excess risk
$
\|f_{\hT} - \fstar\|_2^2 = \sO_{\P}\pr{\Lambda^2 \, n^{-\frac{2}{3 d}}}~.
$
\paragraph{Empirical stopping rule.}
The stopping rule of \cite{dieuleveut2016nonparametric} already gives us optimal rate for a known non-empirical tuning of the algorithm.
To be clear, this is possible because we know the decay rate of kernel spectrum in the case of learning with shallow \ac{ReLU} neural networks.
On the other hand, consider going beyond the setting of our paper where we might want to extrapolate to the situation where a more complex neural network is used (for instance, a multilayer one), while we might still believe that \ac{NTK} approximation is a reasonable model to capture learning.
In such case the decay rate of eigenvalues is unknown and we might want to have an adaptive stopping rule which still learns at the best possible rate.

Here we discuss one of the many data-dependent stopping rules designed for this purpose (see \citep{celisse2021analyzing} for a comprehensive overview).
Specifically, we consider the rule of \ac{RWY} \citep{raskutti2014early}, which empirically controls the complexity of a predictor class within the \ac{RKHS}.\footnote{Technically, the rule controls an upper bound on the empirical localized Rademacher complexity of the ball within the \ac{RKHS}~\citep{bartlett2002localized,wainwright2019high}.}
Computationally, this boils down to having access to the eigenvalues of the \ac{NTK} matrix (or say its empirical approximation, the \ac{NTF} Gram matrix) $\lambda_1 \geq \lambda_2 \geq \ldots \geq \lambda_n$.
In particular, the rule suggests stopping \ac{GD} after $\hT$ steps where
\begin{align}
  \label{eq:stopping}
  \hT = \min\cbr{t \in \mathbb{N} \bmid \hat{\sR}(1/\sqrt{\eta t}) > (2 e \sigma \eta t)^{-1}} - 1~,
  \quad\hat{\sR}(x) = \pr{\frac{1}{n} \sum_{i=1}^n (x^2 \vee (\lambda_i/n))}^{\frac12}~.
\end{align}
Intuitively, in the above $\hat{\sR}(x)$ captures the complexity of the ball within the \ac{RKHS}: The smaller is the argument $x$, the simpler is the class of considered functions.%
\footnote{Since functions on the \ac{RKHS} are defined in the span of the kernel eigenbasis (see \Cref{sec:rkhs}), simpler functions lie in the low-dimensional subspace spanned by the largest eigenvalues. Note that parameter $x$ in $\hat{\sR}(x)$ determines the cut-off for eigenvalues.}
Thus, the number of steps is selected to restrict the solution of \ac{GD} to the simplest possible functions given the data, where the simplicity is chosen with respect to the noise rate $\sigma^2$.
Note that $\hT \to \infty$ as $\sigma \to 0$, which suggests that in the noise-free regime \ac{GD} interpolates on the training sample.
This rule requires the knowledge of $\sigma^2$, which is not unreasonable and can be estimated on a hold-out set.
As a sanity check, we verify that this rule is still able to recover the optimal rate in our Lipschitz setting.
\citet{raskutti2014early} (see also \Cref{sec:rwy-result}) showed that \ac{GD} using this rule is guaranteed with high probability to achieve optimal excess risk with
$\kexcess \leqC n^{-\frac{d}{d + 2}}$, and
so we recover the optimal rate for Lipschitz functions of \cref{eq:rate} as well.
\subsection{Related work}
\label{sec:related-new}
It is well known that the penalized \ac{ERM} procedure defined with respect to shallow neural networks,
that is $\bhtheta \in \argmin_{\btheta}\{ \hL(\btheta) + \mathrm{pen}(\btheta) \}$, matches nonparametric rate $\sO_{\P}(n^{-\frac{2}{2+d}})$ for the appropriate choice of penalization \citep{devroye1996probabilistic,gyorfi2006distribution}.
Here, a standard penalty function is a norm of parameters (usually $\ell^2, \ell^1$ or $\ell^{\infty}$) with a carefully tuned magnitude factor depending on $(\fstar, \sigma^2, S)$.
In this work we focus on \ac{GD} rather than \ac{ERM}, and show that \ac{GD} used to minimize empirical risk can achieve an optimal rate without knowledge of unknown parameters.
Several other recent works have also focused on excess risk of \ac{GD} for training shallow neural networks.

\citet{hu2021regularization} showed that for the regression functions within \ac{RKHS} induced by \ac{NTK}, \ac{GD} with a carefully tuned $\ell^2$ penalization
can achieve rate of excess risk
\begin{align*}
  \sO_{\P}\pr{ n^{-\frac{d}{2 d - 1}} }
  \qquad \text{as} \qquad n \to \infty~.
\end{align*}
Note that the class of regression functions ($\ell^2$-ball on \ac{RKHS}) considered by \citet{hu2021regularization} is \emph{much smaller} than the class of $\fstar$ we analyze here, hence the rate has a much better exponential dependence on $d$.
It is also known that the minimax optimal rate for learning in such classes
is $\Theta(n^{-\frac{d}{d+2}})$ \citep{steinwart2009optimal,bach2017breaking}.

Several recent papers went beyond \ac{RKHS} setting and considered less regular regression functions.
To this end, \citet{bietti2022learning} analyzed a so-called \emph{single-index model}: Here, the regression function is given as $\fstar_{\mathrm{si}}(\bx) = g(\ip{\btheta^{\star}, \bx})$, where $g : \reals \to \reals$ is a twice-differentiable Lipschitz `link' function and $\btheta^{\star}$ are unknown parameters.
Instead of \ac{GD} they focused on the early-stopped \emph{Gradient Flow} algorithm minimizing $\ell^2$-penalized empirical risk, and showed that the rate of the excess risk is at best $n^{-\frac14}$.
In the single-index model, thanks to the additional structure imposed on $\fstar$, the optimal rate of excess risk is much faster than the rate for the general case we consider in this paper.
In such case, it is known that the constrained \ac{ERM} procedure attains a minimax optimal rate of order $n^{-\frac{2p}{2p+1}}$ when link function is $p$-times differentiable~\citep[Chap. 22]{gyorfi2006distribution}, which suggests that the rate of \citet{bietti2022learning} is potentially improvable till $n^{-\frac{4}{5}}$.

In similar line of work, \citet{damian2022neural} considered a generalization of the above known as the \emph{multi-index model}, where the regression function is
$\fstar_{\mathrm{mi}}(\bx) = g(\ip{\btheta_1^{\star}, \bx} + \cdots + \ip{\btheta_r^{\star}, \bx})$ and $g$ is a polynomial of degree $p$.
Similarly to \citet{bietti2022learning} they analyzed minimization of $\ell^2$-penalized empirical risk, but this time, by the early-stopped \ac{GD}.
They showed excess risk of order $\sqrt{\frac{d r^p}{n}}$ (for a network of width $m \geqC r^p$),
which is near-optimal for the considered polynomial link functions.
Surprisingly, the rate appears to be much faster than one would expect in the nonparametric setting, however it is worth noting that the polynomial regularity assumption might be essential here.
At the same time,
it is known that for $p$-times differentiable link functions in multi-index models, the minimax optimal rate is of order $n^{-\frac{2 p}{2 p + r}}$~\citep{kohler2016nonparametric}, i.e., exhibiting exponential dependence on $r$.

All these works employ $\ell^2$ penalization (sometimes in conjunction with early stopping), however in a practical setting neural networks are rarely trained with it and yet while trained without penalization they still demonstrate ability to perform well on the unseen sample even in the presence of noise~\citep{zhang2021understanding}.
At the same time practitioners often do not run \ac{GD} methods until \emph{nearly}-zero empirical risk, but rather track the performance on a held-out validation sample and stop the training \emph{early} when an approximate minimum on the validation sample has been reached. %

\subsection{Additional discussion, limitations, and future directions}
\paragraph{Learnability of $\fstar$ using smooth activation functions.}
Regression functions we consider here are non-differentiable, Lipschitz, and bounded.
Numerous works in nonparametric statistics showed that the penalized (or constrained) \ac{ERM} procedure defined with respect to shallow and deep neural networks with \emph{smooth} sigmoidal activation functions, such as sigmoid $x \mapsto e^x/(1+e^x)$, is able to achieve consistency with minimax optimal rates~\citep{gyorfi2006distribution,kohler2005adaptive,kohler2016nonparametric}.
Note that these results are approximation-theoretic, and it is interesting to ask whether optimal rates can also be achieved by an efficient algorithm such as \ac{GD} while employing such smooth activation functions.
Approaching this question from the viewpoint of \ac{NTK}, we need to understand the regularity of the resulting \ac{RKHS} to be able to say how well can be approximate $\fstar$.
To this end, \citet[Sec. 4.2]{wu2022spectral} showed that the \ac{NTK} corresponding to any fixed smooth activation function with bounded derivatives has an \emph{exponentially} decaying spectrum.
Now, consider the approximation error of $s$-Sobolev classes by such \ac{RKHS}:
Namely, it was shown~\citep[Corollary 6.9]{cucker2007learning} that for the \ac{RKHS} with exponential spectral decay of order $\exp(-k^{\alpha})$, we have \emph{logarithmically} decaying approximation error~,
$
  \inf_{\|h\|_{\sH}^2 \leq R}\|h - \fstar\|_2 \leqC (\ln R)^{-c(\alpha) \, s}
$
with $R \geq 1$ and constant $c(\alpha)$ depending only on $\alpha$.
In other words \ac{GD} would have to find a function with exponential $\ell^2$-norm in order to reach a reasonable approximation error, while this would essentially prevent us from controlling the estimation error (see also \citet[Example 13.21]{wainwright2019high} for a related discussion).
This suggests that \ac{GD} is either unable to consistently learn Lipschitz functions with smooth sigmoidal activations or the \ac{NTK} theory is too weak to capture such a learning scenario.
\paragraph{Relationship to analysis of Sobolev spaces.}
The space of Lipschitz non-differentiable and bounded functions considered here can be regarded as an extreme case of Sobolev spaces, namely it is equivalent to $W^{1,\infty}$, see for instance \citep{adams2003sobolev,zadorozhnyi2021online}.
Such Sobolev spaces cannot be embedded into \ac{RKHS} and one needs to control the approximation error, which leads to the excess risk of order $n^{-\frac{2}{2+d}+\ve}$ for an arbitrarily small $\ve > 0$.
Here to show \cref{eq:pre-rate}, we avoid the Sobolev space argument and instead we control the approximation error using theorem of \cite{bach2017breaking} (see \Cref{prop:lip_approx}).
\paragraph{Improved rates for $\fstar$ with more structure.}
Despite being minimax optimal, the rate $n^{-\frac{2}{2+d}}$ is extremely slow for high-dimensional problems: the manifestation of the \emph{curse of dimensionality} is unavoidable unless we impose additional structure on the regression function.
In this work we focus on the general Lipschitz non-differentiable regression functions.
The nonparametric literature often considers some notion of additional regularity imposed on $\fstar$, such as $p$-times differentiability (or a closely related H\"older condition), which allows to establish minimax optimal rates $n^{-\frac{2p}{2p+d}}$, notably a much faster rate than derived here when $p$ is large.
To this end, \cite{kohler2005adaptive} shows that such rates are indeed possible when neural networks are fitted by the \ac{ERM} procedure with carefully chosen constraints on the parameters.
It is an interesting possibility to show this kind of adaptivity to the higher-order smoothness of the regression function for a practical algorithm such as \ac{GD} and its stochastic variants.

Another possible way to beat the curse of dimensionality is to assume that the regression function is only sensitive to the $r$-dimensional subspace of the input space as done in the multi-index model.
As discussed before, \citep{damian2022neural} established an interesting, essentially parametric rate, for polynomial link functions.
At the same time, the literature~\citep{bauer2019deep,kohler2021rate} considering minimizers of the empirical risk constrained to some class of neural networks, shows that the rate $n^{-\frac{2}{2+r}}$ is achievable for Lipschitz link functions.
\ifx\spacecut\undefined
\paragraph{Establishing \cref{eq:rate} for smaller networks.}
The main result of \cref{eq:rate} holds for networks of a polynomial width, that is $m \geq \poly_6\pr{(1+1/d) n}$ and in addition the width has to scale at least as $\poly_8(n/d)$ in order to control the gap we pay for predictions by the neural network instead of learning on \ac{RKHS} directly.
The requirement on the overparameterization we have here matches the one of \citep{du2018gradient}.
Several other optimization-themed results have established convergence of $\hL(\cdot)$ for shallow neural networks under smaller overparameterization requirement, for instance $\poly_2(\cdot)$ \citep{oymak2020towards} (for smooth activations), $\poly_{3/2}(\cdot)$ \citep{song2021subquadratic} (for smooth activations), and linear order \citep{nguyen2021proof,razborov2022improved}.
All of these works are concerned with convergence to global minima of the empirical risk, while in our work we require to control the aforementioned prediction gap.
Moreover these works sometimes require sample-size-dependent tuning of the step size, initialization scale, and normalization of targets, while here these are constant.
At the same time, some other works have looked at the prediction gap in the context of generalization~\citep{arora2019fine,hu2021regularization}, however they where essentially built on the proof of \citep{du2018gradient} and arrived at orders similar to ours.
While the goal of the current work is tangential to the optimization-themed results discussed above, it is definitely an interesting open problem whether optimal nonparametric rates can be achieved by \ac{GD} with much smaller width requirement.
To this end, some approximation-theoretic results have demonstrated that the constrained \ac{ERM} procedure fitting \emph{deep} neural networks can indeed achieve small approximation error with relatively small number of parameters (sometimes linear in the sample size) \citep{bauer2019deep,kohler2021rate,shen2022deep}.
However, it is unclear whether the same can be done for \ac{GD}-trained deep neural networks.
\fi
\subsection{Organization}
In \Cref{sec:sketch} we present the sketch of the analysis that leads to \cref{eq:pre-rate}.
In \Cref{sec:main_results} we present assumptions and complete statements of the theorems.
We showcase examples of applying our results to RWY stopping rule in \Cref{sec:examples}.
All the proofs are presented in \Cref{sec:proofs}. %
\section{Definitions and notation}
In this section we introduce some basic definitions and notation.
Let $\mathbb{S}^{d-1} = \cbr{\bx \in \reals^{d} ~:~ \|\bx\|_2 = 1}$ be the $\ell^2$-norm unit sphere centered at the origin.
Symbol $\poly_k(\cdot)$ stands for a polynomial of degree-$k$.
Throughout the paper $\ln_+(x) = \max(\ln(x),0)$.

Concatenation of vectors is denoted by parentheses, that is $(\bw_1, \ldots, \bw_m) = [\bw_1\tp, \ldots, \bw_m\tp]\tp$.
A vector norm $\|\cdot\|$ is understood as Euclidean norm, while $\|\bx\|_{\infty} = \max_i |x_i|$.
For a \ac{PSD} matrix $\bM$, the weighted inner product is understood as $\ip{\bx, \by}_{\bM} = \ip{\bx, \bM \by}$, and so a weighted semi-norm is then $\|\bx\|_{\bM}^2 = \ip{\bx, \bx}_{\bM}$.
For a matrix $\bM$, $\|\bM\|_{\mathrm{op}}$ denotes its spectral norm while $\|\bM\|_F$ is its Frobenius norm.

For some functions $f, g : \Xdom \to \reals$ we define an empirical inner product which is always taken with respect to the inputs, that is $\ip{f, g}_n = \frac1n \pr{f(\bx_1) g(\bx_1) + \dots + f(\bx_n) g(\bx_n)}$.
Similarly, we have $\|f\|_n^2 = \ip{f,f}_n$.
We also define $\ell^2$ population semi-norm as
$
  \|f\|_2
  = \int_{\Xdom} f(\bx)^2 \diff P_X ~.
$

 \subsection{Basics of reproducing kernel Hilbert spaces}
\label{sec:rkhs}
We recall some basics of \acf{RKHS}.
A Hilbert space $\sH \subset L^2(P_X)$ is a family of functions $f : \Xdom \to \reals$ for which $\|f\|_2 < \infty$ and for which we have an associated inner product $\ip{\cdot, \cdot}_{\sH}$ under which $\sH$ is complete.

A function $\kappa : \Xdom \times \Xdom \to \reals_+$ is called a Mercer kernel if it is continuous, symmetric, and \ac{PSD} in the sense that $\sum_{i,j} \alpha_i \alpha_j \kappa(\bx_i, \bx_j) \geq 0$ for any $\{\bx_i\}_{i=1}^n \subset \Xdom$, $\balpha \in \reals^n$, and any $n \in \mathbb{N}$.
Without loss of generality we will assume that $\sup_{\bx \in \Xdom} \kappa(\bx, \bx) \leq 1$.
Given a Mercer kernel, one can construct an associated \ac{RKHS}
such that for each $\bx \in \Xdom$, the function $\kappa(\bx, \cdot)$ belongs to $\sH$ and a \emph{reproducing} relation holds, that is for all $f \in \sH$, $f(\bx) = \ip{f, \kappa(\bx, \cdot)}_{\sH}$.
Mercer's theorem \citep{cucker2007learning} claims that under suitable conditions on $\kappa$, we have a spectral decomposition
\begin{align*}
  \kappa(\bx, \btilx) = \sum_{i=1}^{\infty} \mu_i \Phi_i(\bx) \Phi_i(\btilx) \qquad (\bx, \btilx \in \Xdom)
\end{align*}
where $\mu_1 \geq \mu_2 \geq \ldots \geq 0$ are eigenvalues of a kernel and $\Phi_1, \Phi_2, \ldots$ are eigenfunctions which form an orthonormal basis in $L^2(P_X)$.
\section{Sketch of the Analysis}
\label{sec:sketch}
In this section we sketch the proof that leads to \cref{eq:pre-rate}.
The first step of the analysis is to relate predictions of a \ac{GD}-trained neural network $f_t$ to those of a \ac{GD}-trained \ac{KLS} predictor $f\ntk_t$ (see \Cref{def:ntrf-ntk-predictors}).
The cost of this step is a polynomial overparameterization of a neural network with respect to the sample size, namely for any $t \in \mathbb{N}$, with high probability over initialization and inputs (\Cref{thm:f-coupling} and \Cref{rem:lambda}),%
\footnote{Similar `coupling' results are given for the Gradient Flow algorithm and smooth activation functions in \citep{bartlett2021deep}, our analysis critically requires ReLU activation. \citep{arora2019fine,hu2021regularization} implicitly involve coupling results for \ac{ReLU}, while in \Cref{thm:f-coupling} we state explicit result.}
\begin{align*}
  \sup_{\bx \in \Xdom} |f_t(\bx) - f\ntk_t(\bx)|^2 \leq \frac{\poly_{4}\pr{n /d}}{\sqrt{m}}~.
\end{align*}
At this point we can exploit any analysis of \ac{GD} operating on \ac{RKHS} to control the excess risk of $f\ntk_t$.
However, our final goal is to learn a bounded Lipschitz function $\fstar$ which does not belong to the \ac{RKHS}.
To this end, we require an approximation theorem which tells us how well some function belonging to a `large' ball in the \ac{RKHS} approximates such a Lipschitz function:
\paragraph{\Cref{prop:lip_approx} \citep[Proposition 6]{bach2017breaking} (informal).}\emph{
  \label{prop:lip_approx}
  Let $\fstar$ be a $\Lambda$-Lipschitz bounded function on a unit sphere.
  For $R = \Omega(1)$, there exists $h \in \sH$ with $\|h\|_{\sH}^2 \leq R$ such that
  \begin{align*}
    \sup_{\bx \in \Xdom} |\fstar(\bx) - h(\bx)| \leq A(R), \qquad A(R) \leqC \Lambda^{\frac{d}{d-2}} \pr{\sqrt{R}}^{-\frac{2}{d-2}}~.
  \end{align*}
}
The Lemma above tells us that learning in a sufficiently large ball in the \ac{RKHS} allows us to approximate $\fstar$  well by some approximator $h \in \sH$.
Thus, a missing link is to demonstrate that $f\ntk_t$ actually learns approximator $h$ in such a ball.
Note that this fact is not immediate, since $f\ntk_t$ is being trained given targets generated by $\fstar$, but not $h$.
This gap is again controlled through \Cref{prop:lip_approx}: namely, we introduce a sequence of virtual \ac{GD}-trained \ac{KLS} predictors $(\tilde{f}\ntk_s)_{s=0}^t$ trained on a sample
\begin{align*}
  \tilde{S} = (\bx_i, \tilde{y}_i)_{i=1}^n~, \qquad \tilde{y}_i = h(\bx_i) + \ve_i~, \qquad \|h\|_{\sH}^2 \leq R~.
\end{align*}
Now, a simple application of \ac{GD} update for \ac{KLS}, together with \Cref{prop:lip_approx} gives $\|f\ntk_t - \tilde{f}\ntk_t\|_n^2 \leq A(R)^2$ (see \Cref{lem:fntk-ftilntk-gap-sample}).
At the same time, a uniform convergence argument allows to extend this to the random design setting to get
$\|f\ntk_t - \tilde{f}\ntk_t\|_2^2 \leqC (A(R)^2 + \frac{1}{\sqrt{n}}) $ with high probability (see \Cref{lem:gap-fk-ftilk}).
At this point our analysis can be summarized by the following error decomposition.
For $t \in \mathbb{N}$,
\begin{align*}
  \|f_t - \fstar\|_2^2
  &\leqC
    \|f_t - f\ntk_t\|_2^2
    +
    \|f\ntk_t - \tilde{f}\ntk_t\|_2^2
    +
    \|\tilde{f}\ntk_t - h\|_2^2
    +
    \|h - \fstar\|_2^2 \nonumber\\
  &\leqC
    \frac{\poly_{4}\pr{n/d}}{\sqrt{m}}
    +
    \pr{A(R)^2 + \frac{\Cr{gap}}{\sqrt{n}}}
    +
    \|\tilde{f}\ntk_t - h\|_2^2      
    +
    A(R)^2~.    
\end{align*}
It remains to argue that the excess risk of an early-stopped \ac{GD} $\|\tilde{f}\ntk_t - h\|_2^2$ grows at a desirable rate in the radius $R$.
This would ensure that the ball on \ac{RKHS} is always large enough for the approximation error $A(R)$ to decrease.
Assuming that there exists a function $\kexcessfun : \mathbb{N}^2 \to \reals_+$, such that $\|\tilde{f}\ntk_{\hT} - h\|_2^2 \leq (1+R) \, \kexcess$ (\Cref{ass:excess-kernel}), we control such a trade-off by minimizing
\begin{align*}
  R \mapsto A(R)^2 + R \, \kexcess \qquad\qquad (R = \Omega(1))~.
\end{align*}
In turn this proves our main result (\Cref{thm:random-design}):
\begin{align*}
  \|f_{\hT} - \fstar\|_2^2 \leqC
  \frac{\poly_{4}\pr{n/d}}{\sqrt{m}}
  +
  \pr{\kexcess}^{\frac2d}
  +
  \kexcess~.
\end{align*}
Finally, to demonstrate that the actual resuling rate is minimax optimal we need to consider a concrete stopping rule.
Consider the use of a \ac{RWY} stopping rule \citep{raskutti2014early} (see \Cref{sec:rwy}), which enjoys a minimax optimal rate (on \ac{RKHS}),\footnote{We include an essentially complete proof of this bound in \Cref{sec:proofs-rwy} since \cite{raskutti2014early} considered a special case $R=1$, whereas dependence on $R$ is critical to our analysis.}
that is
$
  \kexcess = \sO_{\P}(n^{-\frac{d}{2+d}})
  $
  as $n \to \infty$.
Plugging the above into our bound and choosing $m \geqC n^{\frac{4}{2+d}} \, \poly_{8}\pr{n/d}$ gives us the desired result
\begin{align*}
  \|f_{\hT} - \fstar\|_2^2 = \sO_{\P}(n^{-\frac{2}{2+d}}) \qquad \text{as} \qquad n \to \infty~.  
\end{align*} %
\section{Main results}
\label{sec:main_results}
Before stating our results, we present some technical preliminaries and assumptions.
The following proposition summarizes some properties of the kernel function used in the analysis.
\begin{proposition}[Kernel induced by \ac{ReLU} activation]
  \label{prop:kernel}
  The following kernel function is called the \ac{NTK} function induced by the ReLU activation:
  \begin{align*}
    \kappa(\bx, \btilx) = (\bx\tp \btilx) \int_{\reals^d} \ind{\bw\tp \bx > 0} \ind{\bw\tp \btilx > 0} \, \sN(\diff \bw \mid \bzero, \bI_d)
    \qquad (\bx, \btilx \in \Xdom)~.
  \end{align*}  
  The following holds for $\kappa$:
  \begin{itemize}
  \item It is a reproducing kernel and has analytic form $\kappa(\bx, \btilx) = (\bx\tp \btilx) (\pi - \arccos(\bx\tp \btilx))$~\citep{cho2009kernel,nips2016daniely,jacot2018neural,scetbon2021spectral}.
  \item $\sup_{\bx, \btilx \in \Xdom} \kappa(\bx, \btilx) \leq 1$.
  \item Eigenvalues of $\kappa$ satisfy $\mu_k \leq \Cl{rkhs-rate} k^{-\frac{d}{2}}$ for $k \in \mathbb{N}$~ \citep{bach2017breaking,scetbon2021spectral}.
  \item The \ac{NTK} matrix is a symmetric matrix $\bK \in \reals^{n \times n}$ with entries $(\bK)_{i,j} = \kappa(\bx_i, \bx_j)$.
  \end{itemize}
  Throughout the paper, $\sH$ is the \ac{RKHS} induced by $\kappa$.
\end{proposition}
  The kernel function $\kappa$ is called the `neural tangent kernel' since it arises by taking the subgradient of the neural network around its initialization~\citep{nips2016daniely,jacot2018neural}.
  In particular, consider a feature map $\bphi(\bx) = \nabla_{\btheta} f_{\btheta}(\bx)\,\big|_{\btheta=\btheta_0}$ random in $\btheta_0$.
  Then it is clear that $\kappa(\bx, \btilx) = \E[\bphi(\bx)\tp \bphi(\btilx) \mid \bx, \btilx]$.
Having established kernel function, we assume the following about matrix $\bK$:
\begin{assumption}[Smallest eigenvalue of the \ac{NTK} matrix]
  \label{ass:ntk}
  Assume that there exists fixed $\lambda_0 > 0$ such that $\P(\lmin(\bK) \geq \lambda_0) \geq 1 - \delta_{\lambda_0}$ and $\delta_{\lambda_0} \in [0,1]$.
\end{assumption}
Clearly, we require the above assumption only in the random design setting.
When we consider a fixed design we have $\lambda_0 = \lmin(\bK)$.
\begin{remark}
  \label{rem:lambda}
    \Cref{ass:ntk} is fairly standard and often satisfied with sample-dependent lower bounds.
    \begin{itemize}
    \item In particular, \citep{du2018gradient} shows that $\lambda_0 > 0$ whenever no two distinct inputs are parallel.
      \ifx\spacecut\undefined
    \item \citet{oymak2020towards} show a distribution-free lower bound, which claims that for inputs satisfying a separation $\min_{i \neq j} (\|\bx_i - \bx_j\| \wedge \|\bx_i + \bx_j\|) \geq \delta$, one has $\lambda_0 \geq \delta / (100 n^2)$.
    \item In a random design setting, \citet[Lemma 5.2]{bartlett2021deep} show that when inputs are sampled from isotropic Gaussian and $n \leq d^{C}$ for some activation function-dependent $C$, $\lambda_0 = \Omega(d)$ with probability at least $1-e^{-\sO(n)}$.
      \fi
    \item In a random design setting, \citet{nguyen2021tight} show that (for a certain well-behaved family of input distributions), $\lmin(\bK) = \Theta_{\P}(d)$ with high probability.
      More precisely, their Theorem 3.2 implies that
      we have $\lmin(\bK) \geq \Cl{lambda0-1} \, \polylog(n,d) d$ with probability at least $1 - n^2 e^{-\sO(\sqrt{d})}$.
    \end{itemize}            
  \end{remark}
  \ifx\spacecut\undefined
  To give a concrete example, we state here the result of \citet{nguyen2021tight}:
  \begin{proposition}[{\citet[Theorem 3.2]{nguyen2021tight}}]
    \label{prop:lambda-zero}
    Let inputs be drawn independently from some distribution on $\Xdom$.
    Then, there exist constants $\Cl{lambda0-1}, \Cl{lambda0-2}, \Cl{lambda0-3} > 0$ such that
  for any even integer $k \geq 2$,
  \begin{align*}
    &\P\pr{\lmin(\bK) \geq \Cr{lambda0-1} \, \chi(k) d}
    \geq 1 - n e^{-\Cr{lambda0-2} \, d} - n^2 e^{-\Cr{lambda0-3} \, d n^{-4/(2 k - 1)}}\\    
    &\qquad\text{where} \quad \chi(k) = \frac{1}{\sqrt{2 \pi}} (-1)^{\frac{k-2}{2}} \frac{(k-3)!!}{\sqrt{k!}}~. \qquad (k \geq 2)
  \end{align*}
\end{proposition}
For instance, taking $k \geq 4 \ln(n)/\ln(d)+1/2$ in \cref{prop:lambda-zero} we have $\lmin(\bK) \geq \Cr{lambda0-1} \, \polylog(n,d) \, d$ with probability at least $1 - n^2 e^{-\sO(\sqrt{d})}$.
\fi
\begin{assumption}[Network width]
  \label{ass:tuning}
  Assume that $\btheta_0$ is sampled as in \Cref{alg:GD}.
  Consider \Cref{ass:ntk}, assume that $\frac{n}{\lambda_0} \geq 1$, assume that $\conf \geq 1$ (where the failure probability over $\btheta_0$ is of order $e^{-\conf}$),
  and assume that the network width satisfies
    \begin{align*}
    m \geq
    \pr{8 \pr{4 B_y^2 \, \frac{n}{\lambda_0} + \sqrt{\conf}}
    + (2 + n)}^4 \pr{\frac{n}{\lambda_0}}^2~.
  \end{align*}
\end{assumption}
Next, the following technical assumption will allow us to reduce excess risk bounds for the early-stopped \ac{GD} learning in \ac{RKHS} to excess risk bounds for learning Lipschitz functions:
\begin{assumption}[Excess risk of the early-stopped \ac{GD} minimizing \ac{KLS}]
  \label{ass:excess-kernel}
  Fix any $h \in \sH$ with $\|h\|_{\sH}^2 \leq R$ and moreover assume that targets are generated as $\tilde{y}_i = h(\bx_i) + \ve_i$.
  Suppose that $\tilde{f}\ntk_{\hT}$ is generated by \ac{GD} with the step size $\eta \in (0, \frac12]$ and the number of steps $\hT$, minimizing the \ac{KLS} objective
  \begin{align*}
    f \mapsto \frac1n \sum_{i=1}^n (\ip{f, \kappa(\bx_i, \cdot)}_{\sH} - \tilde{y}_i)^2 \qquad (f \in \sH)~.
  \end{align*}
  In the random design case, we assume that there exists $\kexcessfun : \mathbb{N}^2 \to \reals_+$ such that
  \[
    \|f\ntk_{\hT} - h\|_2^2 \leq (1 + R) \, \kexcess~.
    \]
    Moreover, in the fixed design case, we assume that there exists $\khexcessfun : \mathbb{N}^2 \to \reals_+$ such that
    \[
      \|f\ntk_{\hT} - h\|_n^2 \leq (1 + R) \, \khexcess~.
    \]
\end{assumption}
\subsection{Random design}
The following \Cref{thm:random-design} shown in \Cref{sec:proof-random-design} establishes relationship between the excess risk of an early-stopped \ac{GD} for learning in \ac{RKHS} and the excess risk of a shallow overparameterized neural network learning in Lipschitz classes.
Before presenting our first result we need the following mild technical assumption:
\begin{assumption}[Decreasing excess risk on \ac{RKHS}]
  \label{ass:excess-kernel-decreasing}
  Let $d > 2$.
  Consider $\kexcess$ given by \Cref{ass:excess-kernel} and assume that for constants $\Cr{gap}, \Cl{lip-approx} > 0$,
  there exists a sufficiently large $n$ such that 
  \begin{align*}
    \kexcess \leq
    \pr{\frac{\Cr{lip-approx}}{\Lambda^2} \vee 1}^{\frac{d}{2-d}} \pr{2 + \frac{\Cr{gap}}{\lambda_0}}~.
  \end{align*}
  \ifx\spacecut\undefined
  where
  constant $\Cl{lip-approx}$ depends only on $d$,\footnote{$\Cr{lip-approx}$ is $C(\alpha, d)$ of \citep[Proposition 6]{bach2017breaking} with $\alpha=0$.}
  and $\Cr{gap}$ explicitly given in \Cref{lem:gap-fk-ftilk}.
  \fi
\end{assumption}
\Cref{ass:excess-kernel-decreasing} is mild because we only consider stopping rules that lead to decreasing excess risk, for instance having $\kexcess = n^{-\alpha}$ with $\alpha > 0$, the assumption
is certainly satisfied for a large enough $n$.
Now we state our main result:
\begin{theorem}
  \label{thm:random-design}
  Consider \Cref{ass:ntk} and
  assume that the width $m$ is chosed according to \Cref{ass:tuning}.
  Assume that the step size of \ac{GD} satisfies $\eta \in (0, \frac12]$ and the number of steps $\hT$ is set according to the stopping rule of \Cref{ass:excess-kernel} with $\kexcess$ given by the assumption.
  Moreover let $\kexcess$ satisfy a technical \Cref{ass:excess-kernel-decreasing}.  
  Then, for any $\conf \geq 1$, having $\delta = 1-(3 + 2 n) e^{-\conf} - \delta_{\lambda_0}$,
  with probability at least $1-\delta$,
  \begin{align*}
    \frac14 \,
    \|f_{\hT} - \fstar\|_2^2
  \leq
   \frac{\poly_{4}\pr{B_y^2 \,\frac{n}{\lambda_0}, \conf}}{\sqrt{m}}
    +
    \tilde{C} \, \Lambda^2 \, \pr{\kexcess}^{\frac{2}{d}}
      +
    \kexcess
    + \frac{\Cr{gap}}{\sqrt{n}}~,
  \end{align*}
  where we have a log-term $\tilde{C}=
  \pr{2 + \Cr{gap} / \lambda_0^2} \,
    \pr{1 + \Cr{lip-approx}^2 \, \ln_+^2 \pr{2 / \kexcess }}$
    with
  $\Cr{gap}$ explicitly given in \Cref{lem:gap-fk-ftilk} and $\poly_{4}(\cdot)$ given in \Cref{thm:f-coupling}.
\end{theorem}
We discuss the role of some terms in the above bound:
\begin{itemize}
\item The term $\poly_{4}\pr{B_y^2 \,\frac{n}{\lambda_0}, \conf} / \sqrt{m}$ appearing in \Cref{thm:random-design} is a price we pay for approximating overparameterized neural network predictor by a \ac{KLS} predictor when both are trained by \ac{GD}.
\item The term $\Lambda^2 \, (\kexcess)^{\frac{2}{d}}$ is essentially a nonparametric rate of learning Lipschitz functions. Namely, when the excess risk on \ac{RKHS} scales as $\kexcess = \sO(n^{-\frac{d}{d+2}})$, we recover the rate $\Lambda^2 \, n^{-\frac{2}{2+d}}$ which is a minimax optimal rate.
\item The above example also suggests that the \Cref{ass:excess-kernel-decreasing} is benign, this translates into requirement $n^{-\frac{d}{d+2}} = \sO(1)$ as $n \to \infty$.
\item Finally, given that $\lambda_0 = \Omega_{\P}(d)$ as discussed in \Cref{rem:lambda},
  $\tilde{C} = \sO_{\P}\pr{\poly(1/d^2, \ln_+^2(n^{\frac{d}{2+d}}))}$.
\end{itemize}
\subsection{Fixed design}
We also present a version of \Cref{thm:random-design} specialized to the fixed design setting, shown in \Cref{sec:fixed-design-proof}.
\begin{theorem}
  \label{thm:fixed-design}  
  Assume that the width $m$ is chosed according to \Cref{ass:tuning}.
Assume that the step size of \ac{GD} satisfies $\eta \leq \frac12$ and the number of steps $\hT$ is set according to \Cref{ass:excess-kernel}.
Then, let $\khexcess$ be given by \Cref{ass:excess-kernel} and suppose that
that the sample size $n$ is large enough to satisfy
\begin{align*}
    \kexcess \leq
    \pr{\frac{\Cr{lip-approx}}{\Lambda^2} \vee 1}^{\frac{d}{2-d}}
  \end{align*}
where $\Cr{lip-approx}$ depends only on $d$.
  Then, with probability at least $1 - 2 (1 + n) e^{-\conf}, \conf \geq 1$ over $\btheta_0$,
  \begin{align*}
    \|f_{\hT} - \fstar\|_n^2 
    &\leq
      \frac{\poly_3(B_y^2 \, \frac{n}{\lambda_0}, \conf)}{\sqrt{m}}
      +
      \tilde{C}' \,
      \Lambda^2 \, \pr{\khexcess}^{\frac{2}{d}}
      +
      \khexcess~,
  \end{align*}
  where we have a log-term $\tilde{C}' = 2 \Big(1 + \Cr{lip-approx}^2 \, \ln_+^2\big(\khexcess^{-\frac12}\big)\Big)$.
\end{theorem}
\section{Implications for concrete stopping rules}
\label{sec:examples}
\subsection{Random design and \ac{RWY} stopping rule}
\label{sec:rwy}
In this section we consider application of \Cref{thm:random-design} to a specific empirical stopping rule \cref{eq:stopping}.
Learning using this stopping rule is not only consistent, but also enjoys optimal rates in \ac{RKHS} as summarized in the following theorem (see \Cref{sec:proofs-rwy}).
\begin{theorem}
  \label{thm:random-design-rwy}
  Fix $h \in \sH$ with $\|h\|_{\sH}^2 \leq R$ and moreover assume that targets are generated as $\tilde{y}_i = h(\bx_i) + \ve_i$.
  Suppose that $\tilde{f}\ntk_{\hT}$ is a \ac{GD}-trained \ac{KLS} predictor given the sample $(\bx_i, \tilde{y}_i)_{i=1}^n$, with
  the step size $\eta$ and the number of steps $\hT$ chosed according to \cref{eq:stopping}.
  Then, there exist universal constants $\Cl{r-decay}, \Cl{random-design-1-wp}, \Cl{random-design-2-wp} > 0$
  such that,
  \begin{align*}
    \|\tilde{f}\ntk_{\hT} - h\|_2^2 \leq  
  9 \Cr{r-decay} \big(1 + (40^2 \sigma^2)^{\frac{d}{d + 2}} \big) \, (1 + R) \, n^{-\frac{d}{d + 2}}
  \end{align*}
  with probability at least  
  $1-\Cr{random-design-1-wp} \, \exp(-\Cr{random-design-2-wp} \, (\sigma^2)^{\frac{d}{d + 2}} n^{\frac{2}{d + 2}})$.
\end{theorem}
We can immediately see that the bound of \Cref{thm:random-design-rwy} satisfies assumptions \ref{ass:excess-kernel} and \ref{ass:excess-kernel-decreasing}, and so
combining \Cref{thm:random-design} and \Cref{thm:random-design-rwy} we have \cref{eq:rate}.
We can see that the width has to satisfy $m \geq \poly_{8}\pr{B_y^2 \,\frac{n}{\lambda_0}, \conf} n^{\frac{4}{2+d}}$ to get the overall rate $n^{-\frac{2}{2+d}}$.
\subsection{Fixed design and \ac{RWY} stopping rule}
\label{sec:rwy-result}
When applied in the fixed design setting, the stopping rule of \cref{eq:stopping} results in data-dependent excess risk bound, which does not explicitly depend on the sample size.
Instead it depends on the following following empirical quantity which depends on the spectrum of the kernel matrix.
  \begin{definition}[Critical radius]
    \label{def:rad}
  Let $\lambda_1 \geq \dots \geq \lambda_n$ be eigenvalues of $\bK$.
  We call the empirical complexity $\hat{\sR}(\cdot)$ given in \cref{eq:stopping}.
  We call the empirical critical radius $\hat{r}$ the smallest positive solution to the inequality
  $
    \hat{\sR}(\sqrt{r}) \leq r / (2 e \sigma)~.
    $
  Quantity
  $\hat{r}$ exists, lies in the interval $(0, \infty)$, and it is unique~\citep[Appendix D]{raskutti2014early}.
\end{definition}
The following lemma gives us the excess risk bound for the early-stopped \ac{GD} in the fixed design sense (see \Cref{app:proofs} for the proof):
\begin{lemma}[{\citet[Theorem 1]{raskutti2014early}}]
  \label{lem:rwy-fixed-design}
  Let $h \in \sH$ with $\|h\|_{\sH}^2 \leq R$ and moreover assume that targets are generated as $\tilde{y}_i = h(\bx_i) + \ve_i$.
  Suppose that $\tilde{f}\ntk_{\hT}$ is a \ac{GD}-trained \ac{KLS} predictor given the sample $(\bx_i, \tilde{y}_i)_{i=1}^n$, with
  the step size $\eta$ and the number of steps $\hT$ chosed according to \cref{eq:stopping}.
  Let the empirical critical radius $\hat{r}$ be as in \Cref{def:rad}.
  Then, there exists a universal constant $\Cl{rwy-var}$ such that,  
  \begin{align*}
    \P_{\bve}\pr{\|\tilde{f}\ntk_{\hT} - h\|_n^2 \leq 2 (R + 5) \hat{r}} \geq 1 - e^{-\Cr{rwy-var} n \hat{r}}~.
  \end{align*}
\end{lemma}
Then, we have a corollary of \Cref{thm:fixed-design} and the above:
\begin{corollary}  
  Assume that $\hat{r}$ satisfies
  $
    \pr{10 \, \hat{r}}^{\frac2d - 1} \geq \big((\Cr{lip-approx}/\Lambda^2) \vee 1 \big)
  $
  where constant $\Cr{lip-approx}$ depends only on $d$.
  Then, under conditions of \Cref{thm:fixed-design},
  with probability at least $1 - 2 (1 + n) e^{-\conf} - e^{-\Cr{rwy-var} n \hat{r}}, \conf \geq 1$ over $(\btheta_0, \ve_1, \ldots, \ve_n)$,
  \begin{align*}
    \frac14 \|f_{\hT} - \fstar\|_n^2 
    &\leq
      \tilde{C}' \,
      \Lambda^2 \, \pr{10 \, \hat{r}}^{\frac{2}{d}}
      +
      10 \, \hat{r}
      +
      \frac{\poly_3(B_y^2 \, \frac{n}{\lambda_0}, \conf)}{\sqrt{m}}~,\\
      \text{where} \qquad \tilde{C}' &= 2 + 2 \Cr{lip-approx}^2 \, \ln_+^2\big(\pr{10 \, \hat{r}}^{-\frac12}\big)~.
  \end{align*}
\end{corollary}
As mentioned before, the resulting excess risk bound above does not involve a sample size-dependent rate, but rather involves on a data-dependent quantity $\hat{r}$.
At this point, it is not possible to say anything more about $\hat{r}$ without making distributional assumptions about inputs.
As a sanity-check, we briefly consider random inputs, namely $\bx_1, \ldots, \bx_n \sim P_X$.
Then, the following proposition then reveals the behavior of $\hat{r}$ as a function of a sample size.
\begin{corollary}
  \label{cor:emp-r}
  Let the critical empirical radius $\hat{r}$ be as in \Cref{def:rad}.
  Then, there exist universal constants $\Cl{radii-1},\Cl{radii-2},\Cr{r-decay} > 0$ such that
  with probability at least $1 - \Cr{radii-1} \exp(-\Cr{radii-2}\Cr{r-decay} (\sigma^2)^{\frac{d}{d + 2}} n^{\frac{2}{d + 2}})$
  \begin{align*}
    \hat{r} \leq \Cr{r-decay} (\sigma^2)^{\frac{d}{d + 2}} n^{-\frac{d}{d + 2}}~.
  \end{align*}  
\end{corollary}
\begin{proof}
  Corollary of \Cref{prop:critical-spectrum}, \Cref{lem:emp-r-to-pop-r} appearing in \Cref{sec:lemmata-rwy},
  and \Cref{prop:kernel}.
\end{proof}
Then, choosing width $m \geq \poly_6(B_y^2 \, \frac{n}{\lambda_0}, \conf) \cdot n^{\frac{4}{2+d}}$ with high probability, for $d > 1$
\begin{align*}
  \|f_{\hT} - \fstar\|_n^2 = \stilO_{\P}\pr{ (1+ \Lambda^2) (\sigma^2)^{\frac{d}{d + 2}} n^{-\frac{2}{d + 2}} } \qquad \text{as} \qquad n \to \infty~.
\end{align*}
This recovers a minimax optimal rate $n^{-\frac{2}{2+d}}$ for learning of Lipschitz functions.
\section{Proofs}
\label{sec:proofs}
In \Cref{sec:basc-relu-facts} and \Cref{sec:relu-convergence} we summarize basic facts and convergence results about \ac{ReLU} neural networks, necessary to show the following `coupling' results between neural network and kernel predictors --- these are shown in \Cref{sec:coupling}.
Proofs of our main results are given in \Cref{sec:fixed-design-proof} and \Cref{sec:proof-random-design} for the fixed and random design cases respectively.
Finally, in \Cref{sec:proofs-rwy} we summarize proofs related to the \ac{RWY} stopping rule.

Throughout proofs (sub-)gradient operator is understood with respect to parameter vector, i.e.\ $\nabla \equiv \nabla_{\btheta}$.
We introduce the following \ac{NTF} and \ac{KLS} related definitions, which will be only required for forthcoming proofs:
\begin{definition}[\ac{NTF}]
  \label{def:ntrf}
  ~\\[-5mm]
  \begin{itemize}
  \item The \ac{NTF} operator for the step $t$ is defined as $\bphi_t(\bx) = \nabla f_{\btheta_t}(\bx)$,
    or equivalently
    \begin{align*}
      \bphi_t(\bx) = \pr{u_1 \mathbb{I}\{\bw_{t,1}\tp \bx > 0\} \bx, \ldots, u_m \mathbb{I}\{\bw_{t,m}\tp \bx > 0\} \bx} \qquad (\bx \in \Xdom)~.
    \end{align*}
    We also use matrix notation when \ac{NTF} is computed on the sample, that is
    \begin{align*}
      \bPhi_t = [\bphi_t(\bx_1), \ldots, \bphi_t(\bx_n)] \in \reals^{d m \times n}~.
    \end{align*}
\item Symmetric square matrix $\bhK_t = \bPhi_t\tp \bPhi_t \in \reals^{n \times n}$ is called the \ac{NTF} Gram matrix.
  \item In the above when no $t$-index is present it is understood that $t=0$, that is $\bphi \equiv \bphi_0$, $\bPhi \equiv \bPhi_0$, $\bhK \equiv \bhK_0$.
  \end{itemize}    
\end{definition}
\begin{definition}[\ac{NTF}/\ac{KLS} predictors]
  \label{def:ntrf-ntk-predictors}
  Fix $\bx \in \Xdom$.
  Suppose that the step size of \ac{GD} satisfies $0 \leq \eta \leq \frac12$.
  \begin{itemize}
  \item We call a sequence $(\bbartheta_t)_{t=0}^{T-1}$ with $\bbartheta_0 = \btheta_0$, the \ac{NTF}-\ac{GD} iterate sequence, defined as
    \begin{align*}
      \bbartheta_{t+1} = \bbartheta_t - \eta \nabla \hL\rf(\bbartheta_t)~,
      \qquad
      \hL\rf(\btheta) = \frac1n \sum_{i=1}^n \pr{\bphi(\bx_i)\tp (\btheta - \btheta_0) - y_i}^2~,
    \end{align*}
    and moreover a \ac{GD}-trained \ac{NTF} predictor at step $t$ is defined as
    \begin{align*}
      f\rf_t(\bx) = \bphi(\bx)\tp (\bbartheta_t - \btheta_0)~.
    \end{align*}
  \item Similarly, we call a sequence $(\balpha\ntk_t)_{t=0}^{T-1}$ with
    $\balpha\ntk_0 = \bzero$, the \ac{KLS}-\ac{GD} iterate sequence, defined as
    \begin{align*}
      \balpha_{t+1} = \balpha_t - \eta \nabla \hL\ntk(\balpha_t)~,
      \qquad
      \hL\ntk(\balpha) = \frac1n \sum_{i=1}^n \pr{\sum_{i=1}^n \alpha_i \kappa(\bx_i, \bx) - y_i}^2 \qquad (\balpha \in \reals^{n})
    \end{align*}
    and moreover the \ac{GD}-trained \ac{KLS} predictor at step $t$ is defined as
    \begin{align*}
      f\ntk_t(\bx) = \sum_{i=1}^n \alpha_t \kappa(\bx_i, \bx)~.
    \end{align*}
  \end{itemize}
\end{definition}

\subsection{Some facts about shallow ReLU neural networks}
\label{sec:basc-relu-facts}
\begin{proposition}
  \label{prop:basic-loss}
  For any any $\btheta \in \reals^{d m}$ and any $\bx \in \Xdom$ we have
  \begin{itemize}
  \item[(i)] $f_{\btheta}(\bx) = \nabla f_{\btheta}(\bx)\tp \btheta$~,
  \item[(ii)] For any $\bx \in \Xdom$ we have $\|\nabla f_{\btheta}(\bx)\|^2 \leq 1$~,
  \item[(iii)] $\|\nabla \hL(\btheta)\|^2 \leq 4 \hL(\btheta)$~.
  \end{itemize}
\end{proposition}
\begin{proof}
  Observe that
  \begin{align*}
  \nabla f_{\btheta}(\bx) = (u_1 \ind{\bw_1\tp \bx > 0} \bx, \ldots, u_m \ind{\bw_m\tp \bx > 0} \bx) \in \reals^{dm}~.
\end{align*}
  Fact $(i)$ comes by observing that
  \begin{align*}
  f_{\btheta}(\bx) = \sum_{k=1}^m u_k (\bw_k\tp \bx)_+ = \sum_{k=1}^m u_k \ind{\bw_k\tp \bx > 0} \bx\tp \bw_k = \nabla f_{\btheta}(\bx)\tp \btheta~.
\end{align*}
Fact $(ii)$ comes simply by computing the squared norm:
\begin{align*}
  \|\nabla f_{\btheta}(\bx)\|^2 = \frac1m \sum_{k=1}^m \ind{\bw_k\tp \bx > 0}^2 \|\bx\|^2 \leq 1~.
\end{align*}
Now, consider the gradient of the loss
$\nabla (f_{\btheta}(\bx) - y)^2 = 2 (f_{\btheta}(\bx) - y) \nabla f_{\btheta}(\bx)$,
and so
\begin{align*}
  \|\nabla (f_{\btheta}(\bx) - y)^2\|^2 \leq 4 (f_{\btheta}(\bx) - y)^2 \|\nabla f_{\btheta}(\bx)\|^2
  \leq
  4 (f_{\btheta}(\bx) - y)^2~.
\end{align*}
The above combined with Jensen's inequality gives us
\begin{align*}
  \|\nabla \hL(\btheta)\| \leq \frac{2}{n} \sum_{i=1}^n |f_{\btheta}(\bx_i) - y_i|
  \leq 2 \sqrt{\hL(\btheta)}
\end{align*}
and squaring this shows fact $(iii)$.
\end{proof}
\begin{proposition}
  \label{prop:ind-abs}
  For all $\bw, \btilw, \bx \in \reals^d$,
  $\ind{\bw\tp \bx > 0} - \ind{\btilw\tp \bx > 0} \neq 0$ \, $\Longrightarrow$ \, $|\btilw\tp \bx| \leq |(\bw - \btilw)\tp \bx|$.
\end{proposition}
\begin{proposition}[Activation patterns]
  \label{prop:neuron-patterns}
  Assume that $\bu$ and $\btheta_0 = (\bw_{0,1}, \ldots, \bw_{0,m})$ are chosen as in \Cref{alg:GD}.
  Consider the set of indices of neurons that changed their activation patterns on input $\bx$, when $\btheta_0$ is replaced by some parameters $\btiltheta = (\btilw_1, \ldots, \btilw_m)$:
  \begin{align*}
    P(\btiltheta, \bx) = \cbr{k \in [m] \bmid \mathbb{I}\{\btilw_{k}\tp \bx > 0\} - \mathbb{I}\{\bw_{0,k}\tp \bx > 0\} \neq 0} \qquad (\btiltheta \in \reals^{d m}, \bx \in \Xdom)~.    
  \end{align*}
  Then, for $\btiltheta$ whose components satisfy $\max_k \|\btilw_k - \bw_{0,k}\| \leq \rho$ for some fixed $\rho \geq 0$, and
  any $\bx \in \Xdom$, the following facts hold:
  \begin{itemize}
  \item[(i)] For all $k \in P(\btiltheta, \bx)$, $|\bw_{0,k}\tp \bx| \leq \|\btilw_k - \bw_{0,k}\|$.
  \item[(ii)]
    $\E|P(\btiltheta, \bx)| \leq m  \rho$~.
  \item[(iii)] With probability at least $1-2 e^{-\conf}$ for any $\conf > 0$,
    $|P(\btiltheta, \bx)| \leq m \rho + \sqrt{m \conf}$~.
  \end{itemize}
\end{proposition}
\begin{proof}
  Fact $(i)$ is immediate from \Cref{prop:ind-abs} and Cauchy-Schwartz inequality.
  ~\\
  {Proof of $(ii)$.}
  Now,
  \begin{align*}
    \E|P(\btiltheta, \bx)|
    &= \sum_{k=1}^m \E \ind{\mathbb{I}\{\btilw_{k}\tp \bx > 0\} - \mathbb{I}\{\bw_{0,k}\tp \bx > 0\} \neq 0}\\
    &\stackrel{(a)}{\leq} \sum_{k=1}^m \E \ind{|\bw_{0,k}\tp \bx| \leq \rho}\\
            &= m \P\pr{|\bw_{0,1}\tp \bx| \leq \rho}\\
            &\leq \sqrt{\frac{2}{\pi}} \, m \rho
  \end{align*}
  by integration of an absolute value of a Gaussian random variable and where
  $(a)$ comes by \Cref{prop:ind-abs}, Cauchy-Schwartz inequality, and $\max_{k \in [m]}\|\btilw_k - \bw_{0,k}\| \leq \rho$.
  ~\\
  {Proof of $(iii)$.}
  Recall that by the initialization of \Cref{alg:GD}, $(\bw_{0,k})_{k=1}^{m/2}$ are i.i.d. vectors,
while the second half is a copy of the first half.
Consequently, by Hoeffding's inequality
  \begin{align*}
    \P\pr{\sum_{k=1}^{m/2} \ind{|\bw_{0,k}\tp \bx| \leq \rho}
    -\frac{m}{2}\P\pr{|\bw_{0,1}\tp \bx| \leq \rho}
    \leq
    \sqrt{\frac{m \conf}{4}}} \geq 1 - e^{-\conf} \qquad (\conf > 0)
  \end{align*}
  and so using a union bound
  \begin{align*}
    \P\pr{\sum_{k=1}^{m} \ind{|\bw_{0,k}\tp \bx| \leq \rho} - m \P\pr{|\bw_{0,1}\tp \bx| \leq \rho} \leq \sqrt{m \conf}} \geq 1 - 2 e^{-\conf} \qquad (\conf > 0)~.
  \end{align*}
  Using fact $(ii)$ to control $\P(|\bw_{0,1}\tp \bx| \leq \rho)$ completes the proof.
\end{proof}
\begin{corollary}
  \label{cor:feature-drift}
  Under conditions of \Cref{prop:neuron-patterns}, for any $\bx \in \Xdom$, with probability at least $1-2 e^{-\conf}, \conf > 0$,
  \begin{itemize}
  \item[(i)]
    $\|\bphi_{\btiltheta}(\bx) - \bphi(\bx)\|^2 \leq \rho + \sqrt{\frac{\conf}{m}}$~,
  \item[(ii)]
    $(\bphi_{\btiltheta}(\bx) - \bphi(\bx))\tp \btiltheta \leq \rho \pr{\sqrt{m} \rho + \sqrt{\conf}}$~.
  \end{itemize}
\end{corollary}
\begin{proof}
  Fact $(i)$ comes by observing that
  \begin{align*}
    \|\bphi_{\btiltheta}(\bx) - \bphi(\bx)\|^2
    &=
    \frac{1}{m} \sum_{k=1}^m \pr{\mathbb{I}\{\btilw_{k}\tp \bx > 0\} - \mathbb{I}\{\bw_{0,k}\tp \bx > 0\}}^2 \|\bx\|^2\\
    &\leq
      \frac{1}{m} \, |P(\btiltheta, \bx)|\\
    &\leq
      \rho + \sqrt{\frac{\conf}{m}}~. \tag{W.p.\ at least $1-2 e^{-\conf}, \conf > 0$ by \Cref{prop:neuron-patterns}}
  \end{align*}
  Fact $(ii)$ comes by
  \begin{align*}
    (\bphi_{\btiltheta}(\bx) - \bphi(\bx))\tp \btiltheta
    &=
      \sum_{k=1}^m u_k \pr{\mathbb{I}\{\btilw_k\tp \bx > 0\} - \mathbb{I}\{\bw_{0,k}\tp \bx > 0\}} \bx\tp \btilw_k\\
    &\leq \frac{1}{\sqrt{m}} \sum_{k=1}^m \abs{\mathbb{I}\{\btilw_k\tp \bx > 0\} - \mathbb{I}\{\bw_{0,k}\tp \bx > 0\}} |(\btilw_k - \bw_{0,k})\tp \bx| \tag{By \Cref{prop:ind-abs}}\\
    &\leq \frac{\rho}{\sqrt{m}} \sum_{k=1}^m \abs{\mathbb{I}\{\btilw_k\tp \bx > 0\} - \mathbb{I}\{\bw_{0,k}\tp \bx > 0\}} \tag{Cauchy-Schwartz inequality}\\
    &= \frac{\rho}{\sqrt{m}} \, |P(\btiltheta, \bx)|\\
    &\leq \rho \pr{\sqrt{m} \rho + \sqrt{\conf}}. \tag{W.p.\ at least $1-2 e^{-\conf}, \conf > 0$ by \Cref{prop:neuron-patterns}}
  \end{align*}
\end{proof}

\subsection{Convergence of ReLU Network and Parameter Drift}
\label{sec:relu-convergence}
In the following proofs, the Euclidean distance traveled up to step $t$ by a single neuron from its initialization is denoted by
\begin{align*}
  \rho_t = \max_{k \in [m]} \|\bw_{0,k} - \bw_{t,k}\| \qquad (t \in \mathbb{N})~.
\end{align*}
\begin{lemma}[Parameter drift --- {\cite[Corollary 4.1]{du2018gradient}}]
  \label{lem:path}
  Assume that $\hL(\btheta_t) \leq B_y^2 (1 - \tfrac{\eta \lambda_0}{2 n})^t$ for any $t \in \mathbb{N}$ and any $\eta \leq \frac12$.
  Then,
  \begin{align*}
    \rho_{t+1} \leq \frac{1}{\sqrt{m}} \pr{4 B_y^2 \,\frac{n}{\lambda_0}}~.
  \end{align*}
\end{lemma}
The goal of this section is to show that the assumption of \Cref{lem:path} holds, which we will do in \Cref{thm:conv}.
Before that we need some basic facts which follow from \Cref{lem:path}.
Note that \Cref{cor:feature-drift} together with \Cref{lem:path} implies the following.
\begin{corollary}
  \label{cor:feature-drift-m}
  Under conditions of \Cref{prop:neuron-patterns}, for any $\bx \in \Xdom$ and any $t \in \mathbb{N}$, with probability at least $1-2 e^{-\conf}, \conf \geq 1$,
  \begin{itemize}
  \item[(i)]
    $\|\bphi_t(\bx) - \bphi(\bx)\|^2 \leq \frac{1}{\sqrt{m}} \pr{ 4 B_y^2 \,\frac{n}{\lambda_0} + \sqrt{\conf}}$~,
  \item[(ii)]
    $(\bphi_t(\bx) - \bphi(\bx))\tp \btheta_t \leq \frac{1}{\sqrt{m}} \, \pr{4 B_y^2 \,\frac{n}{\lambda_0} + \sqrt{\conf}}$~.
  \end{itemize}
\end{corollary}
\begin{proposition}[Drift of empirical Gram matrix]
  \label{prop:lambda-Kt-K0}
  Assume that the initial parameters $\btheta_0$ are sampled as in \Cref{alg:GD}.
  Under conditions of \Cref{lem:path}, with probability at least $1-2 e^{-\conf}, \conf > 0$ over $\btheta_0$,
  \begin{align*}
  \lmin(\bhK_t)
  \geq \lmin(\bhK) - \frac{2 n^2}{\sqrt{m}} \pr{4 B_y^2 \,\frac{n}{\lambda_0} + \sqrt{\conf}}~.
\end{align*}
\end{proposition}
\begin{proof}
  Abbreviate $\mathbb{I}_{t,k,i} = \mathbb{I}\{\bw_{t,k}\tp \bx_i > 0\}$. Then,
  \begin{align*}
    \|\bhK - \bhK_t\|_{\mathrm{op}}
    &\leq
      \|\bhK - \bhK_t\|_F\\
    &\leq
      \sum_{i,j} |(\bhK)_{i,j} - (\bhK_t)_{i,j}|\\
    &\leq
      \sum_{i,j} \frac1m \sum_{k=1}^m \abs{\mathbb{I}_{0,k,i} \mathbb{I}_{0,k,j} - \mathbb{I}_{t,k,i} \mathbb{I}_{t,k,j}} \abs{\bx_i\tp \bx_j}\\
    &=
      \sum_{i,j} \frac1m \sum_{k=1}^m \abs{(\mathbb{I}_{0,k,i} - \mathbb{I}_{t,k,i}) \mathbb{I}_{0,k,j} + \mathbb{I}_{t,k,i} ( \mathbb{I}_{0,k,j} -  \mathbb{I}_{t,k,j})} \abs{\bx_i\tp \bx_j}\\
    &\leq
      \sum_{i,j} \frac1m \sum_{k=1}^m \abs{\mathbb{I}_{0,k,i} - \mathbb{I}_{t,k,i}}
      +
      \sum_{i,j} \frac1m \sum_{k=1}^m \abs{\mathbb{I}_{0,k,j} -  \mathbb{I}_{t,k,j}}\\
    &\leq
      2 n^2 \, \frac{\max_{i \in [n]} |P(\btheta_t, \bx_i)|}{m} \tag{with $P(\cdot, \cdot)$ defined in \Cref{prop:neuron-patterns}}\\
    &\leq
      2 n^2 \pr{\rho + \sqrt{\frac{\conf}{m}}} \tag{By fact $(iii)$ of \Cref{prop:neuron-patterns} w.p. at least $1-2 e^{-\conf}$}\\
    &\leq
      \frac{2 n^2}{\sqrt{m}} \pr{4 B_y^2 \,\frac{n}{\lambda_0} + \sqrt{\conf}}~. \tag{By \Cref{lem:path}}
  \end{align*}
  Finally, combining the above with Weyl's inequality, namely
$
  |\lmin(\bhK_t) - \lmin(\bhK)|
  \leq
  \|\bhK_t - \bhK\|_{\mathrm{op}}
  $
  completes the proof.
\end{proof}
\begin{proposition}[Concentration of empirical Gram matrix]
  \label{prop:lambda}
  Assume that the initial parameters $\btheta_0$ are sampled as in \Cref{alg:GD} and let the failure probability over $\btheta_0$ be $\delta = 2 n e^{-\conf}$ for any $\conf > 0$.
  Assume that the width obeys $m \geq \frac{64 n^2 \conf}{\lambda_0^2}$.
  Then, $\P_{\btheta_0}(\lmin(\bhK) \geq \frac12 \lambda_0) \geq 1 - \delta$.
\end{proposition}
\begin{proof}
  By Weyl's inequality
  \begin{align*}
    \lmin(\bhK) \geq \lmin(\bK) - \|\bK - \bhK\|_{\mathrm{op}}~.
  \end{align*}
  On the other hand the fact $\|\cdot\|_{\mathrm{op}} \leq \|\cdot\|_F$, Hoeffding's inequality, and the union bound over entries give (see \citep[Lemma 3.1]{du2018gradient}
  \begin{align}
    \label{eq:K-hK-concentration}
    \P_{\btheta_0}\pr{\|\bK - \bhK\|_{\mathrm{op}}
    \leq
    4 n \sqrt{\frac{\conf}{m}} } \geq 1 - 2 n e^{-\conf} \qquad (\conf > 0)~.
  \end{align}
  Then requiring 
  $
    \lambda_0 - 4 n \sqrt{\frac{\conf}{m}} \geq \frac12 \lambda_0
  $
  completes the proof.
\end{proof}
\begin{remark}
  \Cref{prop:lambda-Kt-K0} and \Cref{prop:lambda} imply that given a sufficient overparameterization, with high probability, the smallest eigenvalue of any empirical Gram matrix $\bhK_t$ for any $t \in \mathbb{N}$ is strictly positive assuming that $\lambda_0 > 0$.
\end{remark}
The proof of the following theorem essentially follows ideas of \citep{du2018gradient}.
\begin{theorem}[Convergence rate of \ac{GD}]
  \label{thm:conv}
  Assume that the initial parameters $\btheta_0$ are sampled as in \Cref{alg:GD} and
  let the failure probability over $\btheta_0$ be $\delta = 2 (1 + n) e^{-\conf}$ for any $\conf > 0$.
  Assume that the step size obeys $\eta \leq \frac12$.
  Assume that the network width satisfies
  \begin{align*}
    m \geq
    64 \pr{2 \pr{4 B_y^2 \, \frac{n}{\lambda_0} + \sqrt{\conf}}^2
    +
    \pr{4 B_y^2 \, \frac{n}{\lambda_0} + \sqrt{\conf}} (2 + n)}^2 \pr{\frac{n}{\lambda_0}}^2~.
  \end{align*}
  Then, with probability at least $1-\delta$,
  \begin{align*}
    \hL(\btheta_T) \leq B_y^2 \pr{1 - \frac{\eta \lambda_0}{2 n}}^T~.
  \end{align*}  
\end{theorem}
\begin{proof}
  The proof works by induction.
  Assumption of \Cref{lem:path} is the induction hypothesis.
  The base case is immediate (see \Cref{rem:initial}).
  Thus, we need to establish a $t+1$ case
  \begin{align*}
    \hL(\btheta_{t+1}) \leq B_y^2 \pr{1 - \frac{\eta \lambda_0}{2 n}}^{t+1}~.
  \end{align*}
  By the fact $(i)$ of \Cref{prop:neuron-patterns} the vector of predictions of $f_t$ on the training sample can be written as $\bPhi_t\tp \btheta_t$.
  Consequently the \ac{GD} update can be written as
\begin{align*}
  \btheta_{t+1} - \btheta_t = - \frac{2 \eta}{n} \bPhi_t (\bPhi_t\tp \btheta_t - \by)~.
\end{align*}

Now, consider a decomposition of the empirical risk
\begin{align}
  \hL(\btheta_{t+1}) \label{eq:conv-proof-1}
  &=
  \frac1n \|\bPhi_{t+1}\tp \btheta_{t+1} - \by\|^2\\
  &=
    \frac1n \|\bPhi_{t+1}\tp \btheta_{t+1} - \bPhi_t\tp \btheta_{t+1} + \bPhi_t\tp \btheta_{t+1}  - \by\|^2 \nonumber\\
  &\leq
    \underbrace{\frac1n \|(\bPhi_t - \bPhi_{t+1})\tp \btheta_{t+1}\|^2}_{(i)}
    +
    \frac2n \|(\bPhi_t - \bPhi_{t+1})\tp \btheta_{t+1}\| \|\bPhi_t\tp \btheta_{t+1}  - \by\|
    +
    \underbrace{\frac1n \|\bPhi_t\tp \btheta_{t+1}  - \by\|^2}_{(ii)}~. \nonumber
\end{align}
Here, term $(i)$ can be regarded as the `feature change' vector, which is essentially controlled by the number of pattern changes in a feature vector.
Using \Cref{prop:neuron-patterns} we will show that $(i)$ is small whenever width $m$ is large.

Throughout the proof it will be convenient to abbreviate
\[
  \widetilde{n} \df 4 B_y^2 \, \frac{n}{\lambda_0} + \sqrt{\conf}~.
\]
\paragraph{Controlling $(ii)$.}
To analyze $(ii)$ we essentially use the usual \ac{GD} dynamics update
\begin{align*}
  \bPhi_t\tp \btheta_{t+1}  - \by
  &=
  \bPhi_t\tp \btheta_t - \by
  - \frac{2 \eta}{n} \bPhi_t\tp \bPhi_t (\bPhi_t\tp \btheta_t - \by)\\
  &=
    \pr{\bI - \frac{2 \eta}{n} \, \bhK_t} (\bPhi_t\tp \btheta_t - \by)~.
\end{align*}
Observe that $\bI - \frac{2 \eta}{n} \, \bhK_t$ is a \ac{PSD} matrix if we ensure that
$\eta \leq \frac{1}{2 n} \, \lmax(\bhK_t) \leq \frac12$.
Thus taking $\|\cdot\|^2/n$, using Cauchy-Schwartz inequality, and the fact that $x^2 \leq x$ for $x \in [0,1]$,
\begin{align}
  \frac1n \|\bPhi_t\tp \btheta_{t+1}  - \by\|^2
  &\leq \pr{1 - \frac{2 \eta}{n} \, \lmin(\bhK_t)} \hL(\btheta_t) \nonumber\\
  &\leq \pr{1 - \frac{2 \eta}{n} \, \pr{\lmin(\bhK) - \frac{2 n^2}{\sqrt{m}} \, \widetilde{n} }} \hL(\btheta_t) \tag{By \Cref{prop:lambda-Kt-K0}}\\
  &\leq \pr{1 - \frac{\eta \lambda_0}{n} + \frac{4 \eta n \widetilde{n}}{\sqrt{m}} } \hL(\btheta_t) \label{eq:conv-proof-2}
\end{align}
where \cref{eq:conv-proof-2} holds with probability at least $1-2 n e^{-\conf}$ thanks to \Cref{prop:lambda}.
\paragraph{Controlling $(i)$.}
Now it remains to show that the `feature change' term is small.
Abbreviate $\mathbb{I}_{t,k,i} = \mathbb{I}\{\bw_{t,k}\tp \bx_i > 0\}$.
Namely,
\begin{align*}
  \frac1n \|(\bPhi_t - \bPhi_{t+1})\tp \btheta_{t+1}\|^2
  &=
    \frac1n \sum_{i=1}^n \pr{\sum_{k=1}^m u_k \pr{\mathbb{I}_{t+1,k,i} - \mathbb{I}_{t,k,i}} \bx_i\tp \bw_{t+1,k}}^2
\end{align*}
and now consider an inner summand for any $i \in [n]$:
\begin{align*}
  \abs{\sum_{k=1}^m u_k \pr{\mathbb{I}_{t+1,k,i} - \mathbb{I}_{t,k,i}} \bx_i\tp \bw_{t+1,k}}
  &\leq
    \frac{1}{\sqrt{m}} \sum_{k=1}^m \abs{\mathbb{I}_{t+1,k,i} - \mathbb{I}_{t,k,i}} \abs{\bx_i\tp \bw_{t+1,k}}\\
  &\stackrel{(a)}{\leq}
    \frac{1}{\sqrt{m}} \sum_{k=1}^m \abs{\mathbb{I}_{t+1,k,i} - \mathbb{I}_{t,k,i}} \|\bw_{t+1,k} - \bw_{t,k}\|\\
  &\stackrel{(b)}{\leq}
    2 \eta \sqrt{\hL(\btheta_t)} \, \frac{1}{m} \sum_{k=1}^m \abs{\mathbb{I}_{t+1,k,i} - \mathbb{I}_{t,k,i}}\\
  &\stackrel{(c)}{\leq}
    \frac{2 \eta \sqrt{\hL(\btheta_t)}}{m} \pr{
    \sum_{k=1}^m \abs{\mathbb{I}_{t+1,k,i} - \mathbb{I}_{0,k,i}}
    +
    \sum_{k=1}^m \abs{\mathbb{I}_{0,k,i} - \mathbb{I}_{t,k,i}} }\\
  &\leq
    \frac{2 \eta \sqrt{\hL(\btheta_t)}}{m} \pr{|P(\btheta_{t+1}, \bx_i)| + |P(\btheta_t, \bx_i)|}~.
\end{align*}
where $P(\cdot, \cdot)$ is defined in \Cref{prop:neuron-patterns}.
In the above chain of inequalities, $(a)$ comes by \Cref{prop:ind-abs} and Cauchy-Schwartz inequality,
$(b)$ is a basic consequence of \ac{GD} update, namely:
\begin{align*}
  \|\bw_{t+1,k} - \bw_{t,k}\|
  =
  \frac{2 \eta}{n} \lf\|\sum_{i=1}^n (f_t(\bx_i) - y_i) \nabla_{\bw_k} f_t(\bx_i)\rt\|
  \leq
  \frac{1}{\sqrt{m}} \, \frac{2 \eta}{n} \sum_{i=1}^n |f_t(\bx_i) - y_i|
  \leq
  2 \eta \sqrt{\frac{\hL(\btheta_t)}{m}}~,
\end{align*}
while $(c)$ is due to triangle inequality.
Now using fact $(iii)$ of \Cref{prop:neuron-patterns} and
\Cref{lem:path},
\begin{align}
  |P(\btheta_{t+1}, \bx_i)| &\leq m \rho + \sqrt{\conf m} \nonumber\\
             &\leq 4 B_y^2 \, \frac{n \sqrt{m}}{\lambda_0} + \sqrt{\conf m} \label{eq:conv-proof-3}\\
             &= \widetilde{n} \sqrt{m} \nonumber
\end{align}
where \cref{eq:conv-proof-3} holds by \Cref{prop:neuron-patterns} with probability at least $1-2 e^{-\conf}$.
Similarly we get a bound on $|P(\btheta_t, \bx_i)|$.
Thus,
\begin{align}
  \label{eq:phi-diff}
  \frac1n \|(\bPhi_t - \bPhi_{t+1})\tp \btheta_{t+1}\|^2
    \leq
  \frac{16 \eta^2 \widetilde{n}^2}{m} \, \hL(\btheta_t)~.
\end{align}
\paragraph{Putting all together.}
Finally, putting obtained bounds into \cref{eq:conv-proof-1} we get
\begin{align*}
  \hL(\btheta_{t+1})
  &\leq
  \frac{16 \eta^2 \widetilde{n}^2}{m} \, \hL(\btheta_t)
  +
    2 \, \frac{4 \eta \widetilde{n}}{\sqrt{m}} \, \pr{1 - \frac{\eta}{n}\, \lmin(\bhK_t)}^{\frac12} \hL(\btheta_t)
  +
    \pr{1 - \frac{\eta \lambda_0}{n} + \frac{4 \eta n \widetilde{n}}{\sqrt{m}} } \hL(\btheta_t)\\
  &\leq
    \pr{
    1 - \frac{\eta \lambda_0}{n}
    +
    \frac{16 (\eta \widetilde{n})^2}{m}
  +
    \frac{8 \eta \widetilde{n}}{\sqrt{m}}
    +
    \frac{4 \eta n \widetilde{n}}{\sqrt{m}}
    } \hL(\btheta_t)\\
  &\leq
    \pr{
    1-\frac{\eta \lambda_0}{2 n}
    } \hL(\btheta_t)
\end{align*}
by ensuring that the width is chosen to satisfy:
\begin{align*}
  &- \frac{\eta \lambda_0}{n}
  +
  \frac{16 (\eta \widetilde{n})^2}{m}
  +
    \frac{4 \eta \widetilde{n} (2 + n)}{\sqrt{m}}
   \leq - \frac{\eta \lambda_0}{2 n}\\ \quad
  &\Longleftarrow \quad
  - \frac{\lambda_0}{n}
    +
    \frac{8 \eta \widetilde{n}^2}{m}
  +
    \frac{4 \widetilde{n} (2 + n)}{\sqrt{m}}
   \leq - \frac{\lambda_0}{2 n} \tag{Using $\eta \leq \frac12$}\\ \quad
  &\Longleftrightarrow \quad
  \frac{8 \widetilde{n}^2}{m}
  +
    \frac{4 \widetilde{n} (2 + n)}{\sqrt{m}}
    \leq \frac{\lambda_0}{2 n}\\
  &\Longleftarrow \quad
    \pr{
    16 \widetilde{n}^2
  +
    8 \widetilde{n} (2 + n)
    }
    \frac{n}{\lambda_0}
    \leq \sqrt{m}\\
  &\Longleftrightarrow \quad
    \pr{(4 \widetilde{n})^2
    +
    8 \widetilde{n} (2 + n)}^2 \pr{\frac{n}{\lambda_0}}^2 \leq m~.
\end{align*}
Thus, the induction step is completed. %
We complete the proof by applying a union bound for high probability bounds of \cref{eq:conv-proof-2} and \cref{eq:conv-proof-3}.
\end{proof}

\subsection{Coupling results}
\label{sec:coupling}
The main result of this section is the following `coupling' theorem, shown in \Cref{sec:f-coupling}, which gives a bound on the $\ell^{\infty}$-gap between the prediction of a GD-trained shallow neural network and that of the \ac{GD}-trained \ac{KLS} predictor. 
\begin{theorem}
  \label{thm:f-coupling}
  Assume  that $\frac{n}{\lambda_0} \geq 1$.
  Under conditions of \Cref{thm:conv}, for any $t \in \mathbb{N}, \conf \geq 1$, 
  \begin{align*}
    \sup_{\bx \in \Xdom}(f_t(\bx) - f\ntk_t(\bx))^2
    \leq
    \frac{64}{\sqrt{m}} \, \pr{4 B_y^2 \, \frac{n}{\lambda_0} + \sqrt{\conf}}^2
    \pr{
    256 \, \frac{n}{\lambda_0}
    +
    9
    }^2
    +
    \frac{\conf}{m} \, 
    B_y^2 \, \pr{
    24 \, \frac{n}{\lambda_0}
    +
    \frac12
    }^4~.
  \end{align*}
\end{theorem}
\subsubsection{Lemmata for the proof of \Cref{thm:f-coupling}}
In this section we establish a number of coupling results which will eventually lead to \Cref{thm:f-coupling}.
In the following proofs we will occasionally use a recursive relationship:
\begin{align}
  x_{s+1} = a_s x_s + b_s \quad \text{and} \quad x_0 = 0 \tag{For $((a_s,b_s,x_s))_{s=0}^t$ with $(a_s,b_s,x_s) \in \reals^3$}\\
  \Longrightarrow \qquad
  x_t = \sum_{s=1}^t b_s \prod_{k=s+1}^t a_k~. \label{eq:recurse}
\end{align}
We will also use a vector notation for predictions of $f_s, f\rf_s, f\ntk_s$ on inputs $(\bx_1, \ldots, \bx_n)$.
For any $s \in \mathbb{N}$,
\begin{align*}
  \bff_s &= (f_s(\bx_1), \ldots, f_s(\bx_n))~,\\
  \bff\rf_s &= (f\rf_s(\bx_1), \ldots, f\rf_s(\bx_n))~,\\
  \bff\ntk_s &= (f\ntk_s(\bx_1), \ldots, f\ntk_s(\bx_n))~.
\end{align*}
\begin{lemma}[Neural network -- \ac{RF} predictor coupling on the sample]
  \label{lem:net-rf-coupling}
  For any $t \in \mathbb{N}, \conf \geq 1$, under conditions of \Cref{thm:conv},
  \begin{align*}
    \|f_t - f\rf_t\|_n^2
    \leq
    \frac{(32 B_y)^2}{\sqrt{m}} \pr{\frac{n}{\lambda_0}}^2 \pr{4 B_y^2 \, \frac{n}{\lambda_0} + \sqrt{\conf}} \,
    \pr{1 - \frac{\eta}{2 n} \, \lambda_0}^t~.
  \end{align*}
\end{lemma}
\begin{proof}
For now consider the following decomposition on an arbitrary input $\bx$ for step $s \in \mathbb{N}$:
\begin{align}
  &f_{s+1}(\bx) - f\rf_{s+1}(\bx)\\
    &=
      \bphi_{s+1}(\bx)\tp \btheta_{s+1} - \bphi(\bx)\tp \bbartheta_{s+1} \tag{Note that $\bphi(\bx)\tp \btheta_0 = 0$}\\
    &=
      \bphi_{s}(\bx)\tp \btheta_{s+1}
      +
      (\bphi_{s+1}(\bx) - \bphi_{s}(\bx))\tp \btheta_{s+1}
      -
      \bphi(\bx)\tp \bbartheta_{s+1} \nonumber\\
    &=
      \bphi_{s}(\bx)\tp \pr{\btheta_s - \eta \nabla \hL(\btheta_s)}
      +
      (\bphi_{s+1}(\bx) - \bphi_{s}(\bx))\tp \btheta_{s+1}
      -
      \bphi(\bx)\tp \pr{\bbartheta_s - \eta \nabla \hL\rf(\bbartheta_s)} \nonumber\\
    &=
      f_s(\bx) - f\rf_s(\bx)
      +
      \eta \Big(\underbrace{\nabla \hL\rf(\bbartheta_s)\tp \bphi(\bx) - \nabla \hL(\btheta_s)\tp \bphi_s(\bx)}_{\Delta_s(\bx)}\Big)
      +
      \underbrace{(\bphi_{s+1}(\bx) - \bphi_{s}(\bx))\tp \btheta_{s+1}}_{\beta_s(\bx)}~. \label{eq:net-rf-coupling-1}
\end{align}
Now we express $\Delta_s(\bx)$ in terms of $f_s(\bx) - f\rf_s(\bx)$.
Introduce residual terms $r_{s, j} = (f_s(\bx_j) - y_j)$ and $\bar{r}_{s, j} = (f\rf_s(\bx_j) - y_j)$.
Observe that
\begin{align*}
  \Delta_s(\bx)
  &=
    \frac2n \sum_{j=1}^n \pr{\bar{r}_{s,j} \, \bphi(\bx)\tp \bphi(\bx_j) - r_{s,j} \, \bphi_s(\bx)\tp\bphi_{s}(\bx_j)}\\
  &=
    \frac2n \sum_{j=1}^n \pr{(\bar{r}_{s,j} - r_{s,j}) \, \bphi(\bx)\tp \bphi(\bx_j) - r_{s,j} \, (\bphi_s(\bx)\tp\bphi_{s}(\bx_j) - \bphi(\bx)\tp \bphi(\bx_j))}\\
  &=
    \frac2n \sum_{j=1}^n (\bar{r}_{s,j} - r_{s,j}) \, \bphi(\bx)\tp \bphi(\bx_j)
    + \frac2n \sum_{j=1}^n r_{s,j} \,\bphi(\bx)\tp (\bphi(\bx_j) - \bphi_s(\bx_j))
    + \frac2n \sum_{j=1}^n r_{s,j} \, \bphi_{s}(\bx_j)\tp (\bphi(\bx) - \bphi_s(\bx))\\
  &=
    - \frac2n \sum_{j=1}^n \pr{f_s(\bx_j) - f\rf_s(\bx_j)} \bphi(\bx)\tp \bphi(\bx_j)
    +
    \frac2n \sum_{j=1}^n r_{s,j} \bphi(\bx)\tp \pr{\bphi(\bx_j) - \bphi_{s}(\bx_j)}\\
    &+ \frac2n \sum_{j=1}^n r_{s,j} \, \bphi_{s}(\bx_j)\tp (\bphi(\bx) - \bphi_s(\bx))~.
\end{align*}
Now, it is convenient to we write \cref{eq:net-rf-coupling-1} in a vector form over inputs $(\bx_1, \ldots, \bx_n)$, for which we introduce some vector abbreviations:
\begin{align*}  
    \bbeta_s &= (\beta_s(\bx_1), \ldots, \beta_s(\bx_n))~,\\
    \bgamma_s &= (\gamma_s(\bx_1), \ldots, \gamma_s(\bx_n))~, \qquad 
                \gamma_s(\bx) = \frac{2 \eta}{n} \sum_{j=1}^n r_{s,j} \bphi(\bx)\tp \pr{\bphi(\bx_j) - \bphi_{s}(\bx_j)}~,\\
  \bgamma_s' &= (\gamma_s'(\bx_1), \ldots, \gamma_s'(\bx_n))~, \qquad
               \gamma'_s(\bx) = \frac{2\eta}{n} \sum_{j=1}^n r_{s,j} \, \bphi_{s}(\bx_j)\tp (\bphi(\bx) - \bphi_s(\bx))~.
  \end{align*}
  Namely,
  \begin{align*}
    \bff_{s+1} - \bff\rf_{s+1}
    =
    (\bI - \tfrac{2 \eta}{n} \bhK) (\bff_s - \bff\rf_s)
    +
    \bbeta_s + \bgamma_s + \bgamma_s'~.
  \end{align*}
  Unrolling the recursion for $s=t-1,\ldots,1$ by applying \cref{eq:recurse} elementwise:
  \begin{align*}
    \bff_t - \bff\rf_t =
    \sum_{s=1}^t (\bI - \tfrac{2 \eta}{n} \bhK)^{t-s} \pr{\bbeta_s + \bgamma_s + \bgamma_s'}~.
  \end{align*}
  In particular, taking $\ell^2$ norm on both sides, and applying triangle and Cauchy-Schwartz inequalities we get
  \begin{align}
    \|\bff_t - \bff\rf_t\|
    &\leq
      \sum_{s=1}^t \lf\| (\bI - \tfrac{2 \eta}{n} \, \bhK)^{t-s} \rt\|_{\mathrm{op}} \|\bbeta_s + \bgamma_s + \bgamma_s'\| \label{eq:coupling-proof-bhy-bhyrf}\\
    &\leq
      \sum_{s=1}^t (1 - \tfrac{2 \eta}{n} \, \lmin(\bhK))^{t-s} \pr{\|\bbeta_s\| + \|\bgamma_s\| + \|\bgamma_s'\|} \nonumber \\
    &\leq
      \sum_{s=1}^t (1 - \tfrac{\eta}{n} \, \lambda_0)^{t-s} \pr{\|\bbeta_s\| + \|\bgamma_s\| + \|\bgamma_s'\|}~. \tag{By \cref{prop:lambda}}
  \end{align}
  Now we turn our attention to $\|\bbeta_s\|$, $\|\bgamma_s\|$, and $\|\bgamma_s'\|$.
  The bound on the Euclidean norm of $\bbeta_s$ comes by \cref{eq:phi-diff}, namely

\begin{align*}
  \frac{1}{\sqrt{n}} \, \|\bbeta_s\|
  &=
    \frac{1}{\sqrt{n}} \, \|(\bPhi_{s+1} - \bPhi_s)\tp \btheta_{s+1}\|\\
  &\leq
    4 \eta \, \frac{1}{\sqrt{m}} \, \pr{4 B_y^2 \, \frac{n}{\lambda_0} + \sqrt{\conf}} \sqrt{\hL(\btheta_s)}\\
  &\leq
    4 \eta \, \frac{1}{\sqrt{m}} \, \pr{4 B_y^2 \, \frac{n}{\lambda_0} + \sqrt{\conf}} \, B_y (1 - \tfrac{\eta}{2 n} \, \lambda_0)^{\frac{s}{2}}~. \tag{By \cref{thm:conv}}
\end{align*}
Now we give an upper bound on $\|\bgamma_s\|$ by first considering individual terms in the norm.
For any $i \in [n]$,
\begin{align*}
  \frac{1}{(2 \eta)^2} \, \gamma_{s}(\bx_i)^2
  &\leq
    \frac1n \sum_{j=1}^n r_{s,j}^2 (\bphi(\bx_i)\tp \pr{\bphi(\bx_j) - \bphi_{s}(\bx_j)})^2 \tag{Jensen's inequality}\\
  &\leq
    \|\bphi(\bx_i)\|^2 \, \frac1n \sum_{j=1}^n r_{s,j}^2 \|\bphi(\bx_j) - \bphi_{s}(\bx_j)\|^2\\
  &\leq \frac{1}{\sqrt{m}} \, \pr{4 B_y^2 \, \frac{n}{\lambda_0} + \sqrt{\conf}} \hL(\btheta_s) \tag{By \cref{cor:feature-drift-m} w.p.\ at least $1-2 e^{-\conf}, \conf > 0$}\\
  &\leq \frac{1}{\sqrt{m}} \, \pr{4 B_y^2 \, \frac{n}{\lambda_0} + \sqrt{\conf}} B_y^2 (1 - \tfrac{\eta}{2 n} \, \lambda_0)^{s}~. \tag{By \cref{thm:conv}}
\end{align*}
Note that $\gamma'_s(\bx_i)^2$ can be handled in the same way as above.
Thus,
\begin{align*}
  \frac{1}{\sqrt{n}} \, \|\bgamma_s\|
  \leq
  2 \eta \, \frac{1}{\sqrt[4]{m}} \, \pr{4 B_y^2 \, \frac{n}{\lambda_0} + \sqrt{\conf}}^{\frac12} B_y (1 - \tfrac{\eta}{2 n} \, \lambda_0)^{\frac{s}{2}}~,
\end{align*}
and the same bound holds on $\frac{1}{\sqrt{n}} \, \|\bgamma_s'\|$.
Now, turning back to \cref{eq:coupling-proof-bhy-bhyrf} and assuming that $\conf \geq 1$,
\begin{align*}
  \frac{1}{\sqrt{n}} \, \|\bff_t - \bff\rf_t\|
  &\leq
    \sum_{s=1}^t (1 - \tfrac{\eta}{n} \, \lambda_0)^{t-s}
    \Bigg(
    4 \eta \, \frac{1}{\sqrt{m}} \, \pr{4 B_y^2 \, \frac{n}{\lambda_0} + \sqrt{\conf}} \, B_y (1 - \tfrac{\eta}{2 n} \, \lambda_0)^{\frac{s}{2}}\\
    &\qquad\qquad\qquad\qquad+
    4 \eta \, \frac{1}{\sqrt[4]{m}} \, \pr{4 B_y^2 \, \frac{n}{\lambda_0} + \sqrt{\conf}}^{\frac12} B_y (1 - \tfrac{\eta}{2 n} \, \lambda_0)^{\frac{s}{2}}
    \Bigg)\\
  &\leq
  8 \eta B_y \, \frac{1}{\sqrt[4]{m}} \, \pr{4 B_y^2 \, \frac{n}{\lambda_0} + \sqrt{\conf}}^{\frac12} \,
  (1 - \tfrac{\eta}{2 n} \, \lambda_0)^{\frac{t}{2}}
  \sum_{s=1}^t (1 - \tfrac{\eta}{2 n} \, \lambda_0)^{\frac{t}{2}-\frac{s}{2}}  
\end{align*}
where furthermore
\begin{align*}
    \sum_{s=1}^t (1 - \tfrac{\eta}{2 n} \, \lambda_0)^{\frac{t}{2}-\frac{s}{2}}
  &=
    \frac{1 - (1 - \tfrac{\eta}{2 n} \, \lambda_0)^{\frac{t}{2}}}{1 - \sqrt{1 - \tfrac{\eta}{2 n} \, \lambda_0}} \tag{Assuming that $\tfrac{\eta}{2 n} \, \lambda_0 \leq 1$, which is implied by $\eta \leq \frac12$}\\
  &\leq
    \frac{4 n}{\eta \lambda_0}~.
\end{align*}
The proof is now complete.
\end{proof}
\begin{lemma}[\ac{RF} -- \ac{KLS} predictor coupling on the sample]
  \label{lem:rf-ntk-coupling}
  For any $t \in \mathbb{N}$, under conditions of \Cref{prop:lambda},
  \begin{align*}    
    \|f\rf_t - f\ntk_t\|_n^2
    \leq
    64 B_y^2
    \, \frac{\conf}{m} \, \pr{\frac{n}{\lambda_0}}^2~.
  \end{align*}
\end{lemma}
\begin{proof}
  From \Cref{def:ntrf-ntk-predictors} it is evident that
  update rules for \ac{NTF} and \ac{KLS} predictors on the entire training sample can be written as
\begin{align*}
  \bff\rf_{s+1} &= \bff\rf_s - \frac{2 \eta}{n} \, \bPhi\tp \bPhi \pr{\bff\rf_s - \by}~,\\
  \bff\ntk_{s+1} &= \bff\ntk_s - \frac{2 \eta}{n} \, \bK (\bff\ntk_s - \by)~,
\end{align*}
$\bff\rf_0 = \bzero$ and $\bff\ntk_0 = \bzero$.
So,
\begin{align*}
  \bff\rf_{s+1} - \bff\ntk_{s+1}
   & =
    \bff\rf_s - \bff\ntk_s
    +
    \frac{2 \eta}{n} \pr{
    \bK (\bff\ntk_s - \by)
    -
    \bhK \pr{\bff\rf_s - \by}
    }\\
  &=
    \bff\rf_s - \bff\ntk_s
    +
    \frac{2 \eta}{n} \pr{
    (\bK -\bhK) (\bff\ntk_s - \by)
    -
    \bhK \pr{\bff\rf_s - \bff\ntk_s}
    }\\
  &=
    \pr{\bI - \frac{2 \eta}{n} \bhK} \pr{\bff\rf_s - \bff\ntk_s}
    +
    \frac{2 \eta}{n}
    (\bK -\bhK) (\bff\ntk_s - \by)~.
\end{align*}
Unrolling the recursion by \cref{eq:recurse} we have
\begin{align*}
  \bff\rf_t - \bff\ntk_t
  =
  \frac{2 \eta}{n} \sum_{s=1}^t \pr{\bI - \frac{2 \eta}{n} \bhK}^{t-s} (\bK -\bhK) (\bff\ntk_s - \by)~.
\end{align*}
Taking Euclidean norm of both sides, applying triangle and Cauchy-Schwarz inequalities,
\begin{align*}
  \|\bff\rf_t - \bff\ntk_t\|
  &\leq
    \frac{2 \eta}{n} \, \|\bK - \bhK\|_{\mathrm{op}} \sum_{s=1}^t \|\bff\ntk_s - \by\| \pr{1 - \frac{2 \eta}{n} \, \lmin(\bhK)}^{t-s}\\
  &\leq
    \frac{2 \eta}{n} \, \|\bK - \bhK\|_{\mathrm{op}} \sum_{s=1}^t \|\bff\ntk_s - \by\| \pr{1 - \frac{\eta}{n} \, \lambda_0}^{t-s} \tag{By \Cref{prop:lambda}}\\
  &\leq
    \pr{\max_{s \in [t]} \frac{1}{\sqrt{n}} \|\bff\ntk_s - \by\|}
    \frac{2 \eta}{\sqrt{n}} \|\bK - \bhK\|_{\mathrm{op}} \sum_{s=1}^t \pr{1 - \frac{\eta}{n} \, \lambda_0}^{t-s}\\
  &\leq
    B_y \,
    \frac{2 \eta}{\sqrt{n}} \, \|\bK - \bhK\|_{\mathrm{op}} \, \frac{n}{\eta \lambda_0}\\
  &\leq
    \pr{\frac{2 B_y}{\sqrt{n}}}
    4 n \sqrt{\frac{\conf}{m}} \, \frac{n}{\lambda_0}
\end{align*}
where the last inequality in the above holds with probability at least $1 - 2 n e^{-\conf}$ over $\btheta_0$ by \cref{eq:K-hK-concentration}.
Thus,
\begin{align*}  
  \|f\rf_t - f\ntk_t\|_n^2 = \frac1n \|\bff\rf_t - \bff\ntk_t\|^2
  \leq  
  16 \, \frac1n \,\pr{\frac{2 B_y}{\sqrt{n}}}^2
  n^2 \, \frac{\conf}{m} \, \pr{\frac{n}{\lambda_0}}^2~.
\end{align*}
\end{proof}
\begin{lemma}[Neural network -- \ac{RF} parameter coupling]
  \label{lem:param-coupling}
  For any $t \in \mathbb{N}, \conf \geq 1$, under conditions of \Cref{thm:conv},
  \begin{align*}
    \|\bbartheta_t - \btheta_t\|^2
    \leq
    \frac{64}{\sqrt{m}} \, \pr{4 B_y^2 \, \frac{n}{\lambda_0} + \sqrt{\conf}}^2
    \pr{
    32 \, \frac{n}{\lambda_0}
    +
    1
    }^2~.
  \end{align*}
\end{lemma}
\begin{proof}
  For any step $s \in \mathbb{N}$,
  abbreviate $\Delta_s = \bbartheta_s - \btheta_s$
  and
  introduce residual terms $r_{s, i} = (f_s(\bx_i) - y_i)$ and $\bar{r}_{s, i} = (f\rf_s(\bx_i) - y_i)$.
Then, \ac{GD} update gives that
\begin{align*}
  \Delta_{s+1}
  &=
    \Delta_s - \eta \pr{\nabla \hL\rf(\bbartheta_s) - \nabla \hL(\btheta_s)}\\
  &=
    \Delta_s - \frac{2\eta}{n} \sum_{i=1}^n (\bar{r}_{s,i} \bphi(\bx_i) - r_{s,i} \bphi_s(\bx_i))\\
  &=
    \Delta_s + \frac{2\eta}{n} \sum_{i=1}^n (\bar{r}_{s,i} - r_{s,i}) \bphi(\bx_i)
    + \frac{2\eta}{n} \sum_{i=1}^n r_{s,i} (\bphi(\bx_i) - \bphi_s(\bx_i))
\end{align*}
and rewriting the above in a vectorized form,
\begin{align*}
  \Delta_{s+1}
  =
  \Delta_s
  +
  \frac{2 \eta}{n} \, \bPhi (\bff_s\rf - \bff_s) + \frac{2 \eta}{n} \, (\bPhi - \bPhi_s) (\bff_s - \by)~.
\end{align*}
Taking Euclidean norm on both sides and applying triangle and Cauchy-Schwartz inequalities,
\begin{align*}
  \|\Delta_{s+1}\|
  &\leq
  \|\Delta_s\|
  +
    \frac{2 \eta}{n} \| \bPhi (\bff_s\rf - \bff_s) \| + \frac{2 \eta}{n} \| (\bPhi - \bPhi_s) (\bff_s - \by)\|\\
  &\leq
  \|\Delta_s\|
  +
    \underbrace{
    \frac{2 \eta}{n} \| \bPhi \|_{\mathrm{op}} \|\bff_s\rf - \bff_s\|
    }_{(i)}
    +
    \underbrace{
    \frac{2 \eta}{n} \| \bPhi - \bPhi_s \|_{\mathrm{op}} \| \bff_s - \by \|
    }_{(ii)}~.
\end{align*}
Now, using the fact that $\| \bPhi \|_{\mathrm{op}} \leq \| \bPhi \|_{F} \leq \sqrt{n}$ (\Cref{prop:basic-loss}) and \Cref{lem:net-rf-coupling},
\begin{align*}
  (i)
    \leq    
  \frac{32 B_y}{\sqrt[4]{m}} \pr{\frac{n}{\lambda_0}} \pr{4 B_y^2 \, \frac{n}{\lambda_0} + \sqrt{\conf}}^{\frac12} \,
    \pr{1 - \frac{\eta \lambda_0}{2 n}}^{\frac{t}{2}}
\end{align*}
Next,
\begin{align*}
  (ii)
  \leq
  \frac{B_y}{\sqrt[4]{m}} \, \pr{4 B_y^2 \, \frac{n}{\lambda_0} + \sqrt{\conf}}^{\frac12} \,
  \pr{1 - \frac{\eta \lambda_0}{2 n}}^{\frac{t}{2}}
\end{align*}
where we used \Cref{thm:conv} to control $\frac{1}{\sqrt{n}} \, \|\bff_s - \by \|$ and used \Cref{cor:feature-drift} to have
\begin{align*}
  \frac{1}{\sqrt{n}} \| \bPhi - \bPhi_s \|_{\mathrm{op}}
  &\leq \frac{1}{\sqrt{n}} \| \bPhi - \bPhi_s \|_F\\
  &\leq \pr{\rho_s + \sqrt{\frac{\conf}{m}}}^{\frac12} \tag{By \Cref{cor:feature-drift} w.p. at least $1 - 2 e^{-\conf}$}\\
  &\leq \frac{1}{\sqrt[4]{m}} \, \pr{4 B_y^2 \, \frac{n}{\lambda_0} + \sqrt{\conf}}^{\frac12}~. \tag{By \Cref{lem:path}}
\end{align*}
Thus, we arrive at
\begin{align*}
  \|\Delta_{s+1}\|
  &\leq
    \|\Delta_s\|
    +
    \eta \Cr{param-coupling}
    \pr{1 - \frac{\eta \lambda_0}{2 n}}^{\frac{t}{2}}
\end{align*}
where
\begin{align*}
  \Cl[eps]{param-coupling} = \frac{2 B_y}{\sqrt[4]{m}} \, \pr{4 B_y^2 \, \frac{n}{\lambda_0} + \sqrt{\conf}}^{\frac12}
    \pr{
    32 \, \frac{n}{\lambda_0}
    +
    1
    }~.
\end{align*}
Now, unrolling the above recursion for $s=t-1,\ldots,1$ by \cref{eq:recurse},
\begin{align*}
  \|\Delta_t\|
  \leq \eta \Cr{param-coupling} \sum_{s=1}^t \pr{1 - \frac{\eta \lambda_0}{2 n}}^{\frac{t}{2} - \frac{s}{2}}
  \leq \Cr{param-coupling} \, \frac{4 n}{\lambda_0}~.
\end{align*}
Having
\begin{align*}
  \Cr{param-coupling} \, \frac{4 n}{\lambda_0} \leq
  \frac{64}{\sqrt{m}} \, \pr{4 B_y^2 \, \frac{n}{\lambda_0} + \sqrt{\conf}}^2
    \pr{
    32 \, \frac{n}{\lambda_0}
    +
    1
    }^2
\end{align*}
completes the proof.
\end{proof}
\begin{lemma}[\ac{KLS} -- \ac{RF} parameter coupling]
  \label{lem:alpha-coupling}
  Consider $(\bbartheta_s)_{s=0}^{t}$ defined in \Cref{def:ntrf-ntk-predictors}.
  For any $t \in \mathbb{N}$,
  under conditions of \Cref{prop:lambda},
  there exists a sequence of vectors $(\bbaralpha_s)_{s=0}^t \subset \reals^n$ satisfying $\bbartheta_s - \btheta_0 = \bPhi \bbaralpha_s$, and moreover
  \begin{align*}
    \|\balpha_t - \bbaralpha_t\|^2
    \leq
    (24 B_y)^2 \, \frac{\conf}{m} \, \frac{n^{3}}{\lambda_0^4}~.
  \end{align*}
  In addition we have that, for any $t \in \mathbb{N}$,
  $\|\balpha_t\| \leq \frac{B_y \sqrt{n}}{\lambda_0}$~.
\end{lemma}
\begin{proof}
  By the definition of \ac{NTF}-\ac{GD} sequence (see \Cref{def:ntrf-ntk-predictors}),
  $(\bbartheta_s - \btheta_0) \in \text{span}(\bphi(\bx_1), \ldots, \bphi(\bx_n))$ and so there must exist $\bbaralpha_s \in \reals^n$ such that
$\bbartheta_s - \btheta_0 = \bPhi \bbaralpha_s$.
So, we have
\begin{align*}
  &\bff_s\ntk - \bff_s\rf
  = \bK \balpha_s - \bhK \bbaralpha_s
  = (\bK - \bhK) \balpha_s + \bhK (\balpha_s - \bbaralpha_s)\\
  \Longrightarrow \qquad
  &\frac{\lambda_0}{2} \, \|\balpha_s - \bbaralpha_s\|
  \stackrel{(a)}{\leq}
  \lmin(\bhK) \|\balpha_s - \bbaralpha_s\|
  \leq
  \|\bff_s\ntk - \bff_s\rf\|
  +
  \|\bK - \bhK\|_{\mathrm{op}} \|\balpha_s\|
\end{align*}
where $(a)$ holds with high probability by \Cref{prop:lambda} (see also \Cref{rem:lambda}).
By \Cref{lem:rf-ntk-coupling} we have with high probability that
\begin{align*}
  \|\bff_s\ntk - \bff_s\rf\| \leq
  8 B_y
    \, \sqrt{\frac{\conf}{m}} \, \frac{n^{1.5}}{\lambda_0}~.
\end{align*}
Similarly, a high-probability bound on $\|\bK - \bhK\|_{\mathrm{op}}$ comes by \cref{eq:K-hK-concentration}.
Finally,
\begin{align*}
  \|\balpha_t\|
  &\leq \|\balpha_{\infty}\|\\
  &= \frac{2 \eta}{n} \lf\|\sum_{s=0}^{\infty} (\bI - \tfrac{2 \eta}{n} \, \bK^2)^s \bK \by \rt\|\\
  &= \frac{2 \eta}{n} \lf\|\pr{\tfrac{2 \eta}{n} \, \bK^2}^{-1} \bK \by \rt\|\\
  &= \lf\|\bK^{-1} \by \rt\|\\
  &\leq \frac{B_y \sqrt{n}}{\lambda_0}~.
\end{align*}
Thus, putting all together, with high probability
\begin{align*}
  \frac{\lambda_0}{2} \, \|\balpha_s - \bbaralpha_s\|
  \leq
  8 B_y
  \, \sqrt{\frac{\conf}{m}} \, \frac{n^{1.5}}{\lambda_0}
  +
  4 n \sqrt{\frac{\conf}{m}} \, \frac{B_y \sqrt{n}}{\lambda_0}~.
\end{align*}
\end{proof}
\subsubsection{Proof of \Cref{thm:f-coupling}}
\label{sec:f-coupling}
  Consider the following decomposition for any step $t \in \mathbb{N}$ and an arbitrary point $\bx \in \Xdom$:
  \begin{align*}
    \abs{f_t(\bx) - f\ntk_t(\bx)}
    \leq
    \abs{f_t(\bx) - f\rf_t(\bx)} + \abs{f\rf_t(\bx) - f\ntk_t(\bx)}~.
  \end{align*}
  The first part on the r.h.s.\ is handled by invoking definition of $f\rf_t$ and \Cref{prop:basic-loss},
  \begin{align*}
    |f_t(\bx) - f\rf_t(\bx)|
    &=
      \abs{
      (\bphi_t(\bx) - \bphi(\bx))\tp \btheta_t + \bphi(\bx)\tp (\btheta_t - \bbartheta_t)
    }\\
    &\leq
      \frac{1}{\sqrt{m}} \, \pr{4 B_y^2 \,\frac{n}{\lambda_0} + \sqrt{\conf}}
      +
      \frac{8}{\sqrt[4]{m}} \, \pr{4 B_y^2 \, \frac{n}{\lambda_0} + \sqrt{\conf}}
      \pr{
      32 \, \frac{n}{\lambda_0}
      +
      1
      }\\
    &\leq
      \frac{1}{\sqrt[4]{m}} \, \pr{4 B_y^2 \, \frac{n}{\lambda_0} + \sqrt{\conf}}
      \pr{
    256 \, \frac{n}{\lambda_0}
      +
      9
    }
  \end{align*}
  where the first inequality in the above is obtained by \Cref{cor:feature-drift-m}, \Cref{lem:param-coupling}, Cauchy-Schwartz inequality, and the fact that $\|\bphi(\bx)\| \leq 1$ (\Cref{prop:basic-loss}).
  On the other hand, by \Cref{def:ntrf-ntk-predictors} we have
  \begin{align*}
    f\rf_t(\bx) - f\ntk_t(\bx)
    &=
      \bphi(\bx)\tp \bbartheta_t - \sum_{i=1}^n \alpha_{t,i} \kappa(\bx_i, \bx)\\
    &=
      \sum_{i=1}^n \bar{\alpha}_{t,i} \bphi(\bx_i)\tp \bphi(\bx) - \sum_{i=1}^n \alpha_{t,i} \kappa(\bx_i, \bx)
      \tag{$\bbartheta_s - \btheta_0 = \bPhi \bbaralpha_s$ as given in \Cref{lem:alpha-coupling}}\\
    &=
      \sum_{i=1}^n (\bar{\alpha}_{t,i} - \alpha_{t,i}) \bphi(\bx_i)\tp \bphi(\bx) - \sum_{i=1}^n \alpha_{t,i} (\kappa(\bx_i, \bx) - \bphi(\bx_i)\tp \bphi(\bx))\\
    &\leq
      \sqrt{n}\|\bbaralpha_t - \balpha_t\|
      +
      \|\balpha_t\|
      \sqrt{n} \, \max_{i \in [n]} |\kappa(\bx_i, \bx) - \bphi(\bx_i)\tp \bphi(\bx)|~.
  \end{align*}
  In the above, bounds on $\|\bbaralpha_t - \balpha_t\|$ and $\|\balpha_t\|$ come from \Cref{lem:alpha-coupling}.
  The remaining bit is to show that the random feature approximation is close to the kernel function, which is done by
  by the  Hoeffding's inequality (and a union bound).
  In particular, noting that for fixed $\bx, \bx_i \in \Xdom$,
\begin{align*}
  \bphi(\bx)\tp \bphi(\bx_i)
  =
  \frac1m \sum_{k=1}^m \mathbb{I}\{\bw_{0,k}\tp \bx > 0\} \mathbb{I}\{\bw_{0,k}\tp \bx_i > 0\} \, \bx\tp \bx_i
\end{align*}
is a sum of independent random variables bounded by $1/m$,
with probability at least $1 - n e^{-\conf}, \conf > 0$,
  \begin{align*}
    \max_{i \in [n]} |\kappa(\bx_i, \bx) - \bphi(\bx_i)\tp \bphi(\bx)|
    =
    \max_{i \in [n]} |\bphi(\bx)\tp \bphi(\bx_i) - \E[\bphi(\bx)\tp \bphi(\bx_i)]|
  \leq
    \sqrt{\frac{\conf}{2 m}}~.
  \end{align*}
  Putting the above together,
  \begin{align*}
    f\rf_t(\bx) - f\ntk_t(\bx)
    &\leq
      \sqrt{\frac{\conf}{m}} \, 
      B_y \, \frac{n}{\lambda_0} \pr{
      24 \, \frac{n}{\lambda_0}
      +
      \frac12
      }
  \end{align*}
and so
  \begin{align*}
    \abs{f_t(\bx) - f\ntk_t(\bx)}
    \leq
      \frac{8}{\sqrt[4]{m}} \, \pr{4 B_y^2 \, \frac{n}{\lambda_0} + \sqrt{\conf}}
      \pr{
    256 \, \frac{n}{\lambda_0}
      +
      9
    }
    +
    \sqrt{\frac{\conf}{m}} \, 
      B_y \, \frac{n}{\lambda_0} \pr{
      24 \, \frac{n}{\lambda_0}
      +
      \frac12
      }~.
  \end{align*}
\jmlrQED
\subsection{Excess risk analysis: proof idea and common tools}
\label{sec:excess-idea}
\Cref{thm:f-coupling} shown in \Cref{sec:coupling} allows us to relate prediction of an overparameterized \ac{GD}-trained shallow neural network $f_t$ to that of the \ac{GD}-trained \ac{KLS} predictor $f\ntk_t$, on any input $\bx \in \Xdom$.
At this point we can leverage any analysis of \ac{GD} operating on \ac{RKHS} to control its excess risk.
However, our final goal is to learn a bounded Lipschitz function $\fstar$ which does not necessarily belong to the \ac{RKHS} $\sH$.
To this end, to prove subsequent results we will require the following result about approximation of Lipschitz functions:
\begin{lemma}[{Approximation of Lipschitz functions on the ball \citep[Proposition 6]{bach2017breaking}}]
  \label{prop:lip_approx}
  For $R$ larger than a constant $\Cr{lip-approx}$ that depends only on $d$, for any function $\fstar : \reals^d \to \reals$ such that for all $\bx, \btilx \in \mathbb{B}_q^d$, $\sup_{\bx \in \mathbb{B}_q^d} |\fstar(\bx)| \leq \Lambda$ and $|\fstar(\bx) - \fstar(\btilx)| \leq \Lambda \|\bx - \btilx\|_q$, there exists $h \in \sH$, such that $\|h\|_{\sH}^2 \leq R$ and
  \begin{align*}
    \sup_{\bx \in \mathbb{B}_q^d(1)} |\fstar(\bx) - h(\bx)| \leq A(R), \qquad A(R) = \Cr{lip-approx} \Lambda \pr{\frac{\sqrt{R}}{\Lambda}}^{-\frac{2}{d-2}} \ln\pr{\frac{\sqrt{R}}{\Lambda}}~.
  \end{align*}
\end{lemma}
Thus, a missing link is to demonstrate that $f\ntk_t$ actually learns approximator $h$.
Note that this fact is not immediate, since $f\ntk_t$ is being trained given targets generated by $\fstar$, but not $h$.
This gap is again controlled through \Cref{prop:lip_approx}: namely, we introduce a sequence of virtual \ac{GD}-trained \ac{KLS} predictors $(\tilde{f}\ntk_s)_{s=0}^t$ trained on a sample
\begin{align*}
  \tilde{S} = (\bx_i, \tilde{y}_i)_{i=1}^n~, \qquad \tilde{y}_i = h(\bx_i) + \ve_i~, \qquad \|h\|_{\sH}^2 \leq R~.
\end{align*}
In particular, on the training sample, the gap is controlled by the following lemma:
\begin{lemma}
  \label{lem:fntk-ftilntk-gap-sample}
  Let $f\ntk_t, \tilde{f}\ntk_t$ be \ac{GD}-trained \ac{KLS} predictors (\Cref{def:ntrf-ntk-predictors}) given training samples $(\bx_i, y_i)_{i=1}^n$ and $(\bx_i, \tilde{y}_i)_{i=1}^n$ respectively, where $y_i = \fstar(\bx_i) + \ve_i$ and $\tilde{y}_i = h(\bx_i) + \ve_i$ with $\|h\|_{\sH}^2 \leq R$ characterized by \Cref{prop:lip_approx}.
  Then, with $A(R)$ defined in \Cref{prop:lip_approx}, for any $t \in \mathbb{N}$,
  \begin{align*}
    \|f\ntk_t  - \tilde{f}\ntk_t\|_n^2 \leq A(R)~.
  \end{align*}
\end{lemma}
\begin{proof}
  As a first step we characterize \ac{GD} predictions at step $t \in \mathbb{N}$.
  In particular, the \ac{GD} update rule from \Cref{def:ntrf-ntk-predictors} gives us
  \begin{align*}
    \bff\ntk_{s+1} = \bff\ntk_s - \frac{2 \eta}{n} \bK (\bff\ntk_s - \by) \tag{$0 \leq s \leq t-1$}
  \end{align*}
  where recall that $\bff\ntk_s = (f\ntk_s(\bx_1), \ldots, f\ntk_s(\bx_n))$,
  and summing the above and unrolling the recursion using \cref{eq:recurse} we get
  \begin{align}
    \bff\ntk_t
    = \frac{2 \eta}{n} \sum_{s=1}^{t} (\bI - \tfrac{2 \eta}{n} \bK)^{t-s} \by
    = \pr{\bI - (\bI - \tfrac{2 \eta}{n} \bK)^t} \by~. \label{eq:gd-analytical}
  \end{align}
  Similarly, we have 
  \begin{align*}
    (\tilde{f}\ntk_t(\bx_1), \ldots \tilde{f}\ntk_t(\bx_n))
    =
    \pr{\bI - (\bI - \tfrac{2 \eta}{n} \bK)^t} \btily~,
  \end{align*}
  and so for any step $t$,
  \begin{align*}
    \|f\ntk_t  - \tilde{f}\ntk_t\|_n^2
    &=
      \frac1n \lf\|\pr{\bI - (\bI - \tfrac{2 \eta}{n} \bK)^t} (\by - \btily) \rt\|^2 \\ %
    &\leq
      \frac1n \lf\| \bI - (\bI - \tfrac{2 \eta}{n} \bK)^t \rt\|_{\mathrm{op}}^2 \|\by - \btily\|^2 \tag{Cauchy-Schwartz inequality}\\
    &\leq
      \|\fstar - h\|_n^2 \nonumber\\
    &\leq
      A(R)^2~. \tag{By \Cref{prop:lip_approx}} %
  \end{align*}
\end{proof}
To this end the approximation-estimation trade-off is then controlled by a simple proposition shown in \Cref{app:proofs},
\begin{proposition}
  \label{prop:tradeoff}
  For $A(R), \Lambda, \Cr{lip-approx}$ defined in \Cref{prop:lip_approx},
  and any $x,y > 0$,
  assume that relationship $\Lambda^2 \, \pr{\frac{z}{x}}^{\frac{2}{d}-1} \geq (\Cr{lip-approx} \vee \Lambda^2)$ holds.
  Then,
  \begin{align*}
    \min_{R \geq (\Cr{lip-approx} \vee \Lambda^2)} \cbr{x \, A(R)^2 + y \, R}
    \leq
  \pr{1 + \Cr{lip-approx}^2 \, \ln_+^2\Big(\pr{\frac{z}{x}}^{\frac{1}{d}-\frac12}\Big) } \Lambda^2 \, x^{1-\frac2d} \, y^{\frac2d}~.
  \end{align*}
\end{proposition}
\subsection{Proof of \Cref{thm:fixed-design} (fixed design)}
\label{sec:fixed-design-proof}
Let $h \in \sH$ be an approximator of $\fstar$ in $\sup$-norm over $\Xdom$ whose existence is shown in \Cref{prop:lip_approx}.
For any $t \leq \hT$, consider the decomposition
\begin{align*}
  \frac14 \|f_t - \fstar\|_n^2
  &\leq
  \|f_t - f\ntk_t\|_n^2
  +
  \|f\ntk_t  - \tilde{f}\ntk_t\|_n^2
  +
  \|\tilde{f}\ntk_t - h\|_n^2
    +
    \|h - \fstar\|_n^2\\
  &\leq
    \frac{\poly_3(\frac{n}{\lambda_0}, \conf)}{\sqrt{m}}
    +
    A(R)^2
  +
    \|\tilde{f}\ntk_t - h\|_n^2
    +
    A(R)^2
\end{align*}
with probability at least $1 - 2 (1 + n) e^{-\conf}$.
Recall that $\tilde{f}\ntk_t$ is a \ac{GD}-trained \ac{KLS} predictor trained on a sample $(\bx_i, \tilde{y}_i)_{i=1}^n$ with $\tilde{y}_i = h(\bx_i) + \ve_i$ (see \Cref{sec:excess-idea}).
Here:
\begin{itemize}
\item Term $\|f_t - f\ntk_t\|_n^2$ is the cost of approximating prediction of the shallow neural network by the \ac{KLS} predictor on the training sample, and which is bounded by \Cref{lem:net-rf-coupling} together with \Cref{lem:rf-ntk-coupling}, with high probability.
\item Term $\|f\ntk_t  - \tilde{f}\ntk_t\|_n^2$ is reduced to the approximation error between $h$ and $\fstar$ and then controlled using
  \Cref{lem:fntk-ftilntk-gap-sample}.
\item Finally, $\|h - \fstar\|_n^2 \leq A(R)^2$ by \Cref{prop:lip_approx}.
\end{itemize}
By \Cref{ass:excess-kernel} we have $\|\tilde{f}\ntk_{\hT} - h\|_n^2 \leq (1 + R) \, \khexcess$ and so we set $t = \hT$ in the above.
To complete the proof we need to control the tradeoff by approximately minimizing
\begin{align*}
  \sC(R)
  =
  2 A(R)^2
  +
    (1 + R) \, \khexcess~.
\end{align*}
Using \Cref{prop:tradeoff} with $x = 2$ and $y = \khexcess$, we get
\begin{align*}
  \min_{R \geq (\Cr{lip-approx} \vee \Lambda^2)} \sC(R)
  \leq
  2 \pr{1 + \Cr{lip-approx}^2 \, \ln_+^2\Big(\khexcess^{\frac{1}{d}-\frac12}\Big)} \,
  \Lambda^2 \pr{\khexcess}^{\frac{2}{d}}
  +
  \khexcess
\end{align*}
which requires assumption $\pr{\khexcess}^{\frac2d - 1} \geq ((\Cr{lip-approx}/\Lambda^2) \vee 1)$ to hold.
The proof of \Cref{thm:fixed-design} is now complete.
\jmlrQED %
\subsection{Proof of \Cref{thm:random-design} (random design)}
\label{sec:proof-random-design}
In \Cref{sec:random-design-lemmata-1} we present lemmata necessary for the proof of \Cref{thm:random-design}, while its proof is given in \Cref{sec:proof-thm-random-design}.
\subsubsection{Lemmata for the proof of \Cref{thm:random-design}: \ac{RKHS} norms and distances}
\label{sec:random-design-lemmata-1}
\begin{proposition}
  \label{prop:ftil-inf-bound}
  Let $f\ntk_t, \tilde{f}\ntk_t$ be \ac{GD}-trained \ac{KLS} predictors (\Cref{def:ntrf-ntk-predictors}) given training samples $(\bx_i, y_i)_{i=1}^n$ and $(\bx_i, \tilde{y}_i)_{i=1}^n$ respectively.
  Then, for any $t \in \mathbb{N}$,
  \begin{align*}
    \|f\ntk_t\|_{\sH} &=
    \|(\bI - (\bI - \tfrac{2 \eta}{n} \, \bK)^{t}) \by\|_{\bK^{-1}}\\
    \|\tilde{f}\ntk_t\|_{\sH} &=
    \|(\bI - (\bI - \tfrac{2 \eta}{n} \, \bK)^{t}) \btily\|_{\bK^{-1}}\\
    \|\tilde{f}\ntk_t - f\ntk_t\|_{\sH} &=
    \|(\bI - (\bI - \tfrac{2 \eta}{n} \, \bK)^{t}) (\btily - \by)\|_{\bK^{-1}}~.
  \end{align*}
\end{proposition}
\begin{proof}
  By the spectral theorem there exists a sequence of non-negative eigenvalues $\lambda_1 \geq \ldots \geq \lambda_n$,
an orthonormal system $u_1, \ldots, u_n \in \sH$, and orthonormal basis $\bv_1, \ldots, \bv_n \in \reals^n$ such that
we have a linear operator $\hat{\Phi} : \reals^n \to \sH$,
\begin{align*}
  \hat{\Phi} = \sum_{i=1}^n \sqrt{\lambda_i}  u_i \otimes \bv_i~,
\end{align*}
where $a \otimes b$ denotes a tensor product between elements $a,b \in \sH$ such that $(a \otimes b) u = a \ip{b, u}_{\sH}$ for every $u \in \sH$.
In particular, denoting the adjoint of $\hat{\Phi}$ by $\hat{\Phi}^{*}$, we have $\hat{\Phi}^{*} \hat{\Phi} = \bK$.
Given the above, \ac{GD}-trained \ac{KLS} predictors are written as
\begin{align*}
  f\ntk_t = \hat{\Phi} \bK^{-1} (\bI - (\bI - \tfrac{2 \eta}{n} \, \bK)^t) \by~, \qquad
  \tilde{f}\ntk_t = \hat{\Phi} \bK^{-1} (\bI - (\bI - \tfrac{2 \eta}{n} \, \bK)^t) \btily
\end{align*}
Thus, using identity $\hat{\Phi}^{*} \hat{\Phi} = \bK$,
  \begin{align*}
    \|\tilde{f}\ntk_t\|_{\sH}
    =
    \|(\bI - (\bI - \tfrac{2 \eta}{n} \, \bK)^t) \btily\|_{\bK^{-1}}~.
  \end{align*}
  Identity for $\|\tilde{f}\ntk_t - f\ntk_t\|_{\sH}$ comes using the same proof with $\btily$ replaced by $\by - \btily$.
\end{proof}
\begin{lemma}
  \label{lem:gap-fk-ftilk}
  Let $h \in \sH$ with $\|h\|_{\sH}^2 \leq R$ and $A(R)$ be given by \Cref{prop:lip_approx}.
  Let $f\ntk_t, \tilde{f}\ntk_t$ be \ac{GD}-trained \ac{KLS} predictors (\Cref{def:ntrf-ntk-predictors}) given training samples $(\bx_i, y_i)_{i=1}^n$ and $(\bx_i, \tilde{y}_i)_{i=1}^n$ respectively, where $y_i = \fstar(\bx_i) + \ve_i$ and $\tilde{y}_i = h(\bx_i) + \ve_i$.
  Then, for any $t \in \mathbb{N}$, with probability at least $1 - e^{-\conf}, \conf > 0$ over inputs,
  \begin{align*}
    \|f\ntk_t - \tilde{f}\ntk_t\|_2^2
    \leq
    \pr{1 + \frac{\Cr{gap}}{\lambda_0^2}} A(R)^2 + \frac{\Cr{gap}}{\sqrt{n}}
  \end{align*}
  where
  \begin{align*}
    \Cl[log]{gap}
    =
    2 (1 \vee \sqrt{2} \log^{\frac32}(n) \vee 2 \log^3(n) \vee \Cr{sup-rad-2}) \pr{\pr{2 \lambda_0 + 4} \vee \lambda_0 \sqrt{\conf} \vee \conf \vee 1}
  \end{align*}  
  and where $\Cr{sup-rad-2}$ is an absolute constant appearing in \citep[Theorem 1]{srebro2010smoothness}.  
\end{lemma}
\begin{proof}
  Since we already proved a similar result of for the fixed design in \Cref{lem:fntk-ftilntk-gap-sample}, we will show its population analogue through a uniform convergence argument.
  In particular, the proof will require the following localized Rademacher complexity bound:
\begin{theorem}[{\citet[Theorem 1]{srebro2010smoothness}}]
  \label{thm:rad-smooth}
  Let $\sF = \cbr{f : \Xdom \to [0, b]}$ for some $b < \infty$.
  Then with probability at least $1-e^{-\conf}, \conf > 0$, for all $f \in \sF$ simultaneously
  \begin{align*}
    \|f\|_2^2 \leq \|f\|_n^2 + \Cr{sup-rad} \pr{
    \|f\|_n \pr{ \widehat{\mathrm{Rad}}(\sF) + \sqrt{\frac{b \conf}{n}} }
    +
    \widehat{\mathrm{Rad}}^2(\sF)
    +
    \frac{b \conf}{n}
    }~,
  \end{align*}
  where the worst-case empirical Rademacher complexity is defined as
  \begin{align*}
    \widehat{\mathrm{Rad}}(\sF) = \sup_{\bx_1, \ldots, \bx_n \in \Xdom} \E \sup_{f \in \sF} \abs{\frac1n \sum_{i=1}^n \gamma_i f(\bx_i)}~.
    \tag{$\gamma_1, \ldots, \gamma_n \stackrel{\mathrm{iid}}{\sim} \unif\{\pm 1\}$}
  \end{align*}
   where
  $\Cl[log]{sup-rad} = (1 \vee \sqrt{2} \log^{\frac32}(n) \vee 2 \log^3(n) \vee \Cr{sup-rad-2})$
  and where $\Cl{sup-rad-2}$ is an absolute constant appearing in \citep{srebro2010smoothness}.
\end{theorem}
\begin{remark}
  \Cref{thm:rad-smooth} comes from \citep[Theorem 1]{srebro2010smoothness} when taking targets to be $y = 0$ almost surely and assuming that the loss function is $\phi(x,y) = (x-y)^2$.
\end{remark}
Specifically, we consider the class
\begin{align*}
  \sF_{\rho} = \cbr{\bx \mapsto f(\bx) - \tilde{f}(\bx) \bmid f, \tilde{f} \in \sH, \|f - \tilde{f}\|_{\sH} \leq \rho }
  \qquad (\rho \geq 0)
\end{align*}
and note that (see for instance \citep[Lemma 22]{bartlett2002rademacher}),
\begin{align*}
  \widehat{\mathrm{Rad}}(\sF_{\rho})
  \leq
  \frac{2 \rho}{\sqrt{n}}~,
\end{align*}
and
$b = \sup_{\bx \in \Xdom}|f(\bx) - \tilde{f}(\bx)| \leq \|f - \tilde{f}\|_{\sH} \leq \rho$ by Cauchy-Schwartz inequality.
Then applying \Cref{thm:rad-smooth}, for all $f - \tilde{f} \in \sF_{\rho}$ simultaneously, with high probability,
\begin{align*}
    \|f - \tilde{f}\|_2^2 \leq \|f - \tilde{f}\|_n^2 + \Cr{sup-rad} \pr{
    \|f - \tilde{f}\|_n \pr{ \frac{2 \rho}{\sqrt{n}} + \sqrt{\frac{\rho \conf}{n}} }
    +
    \frac{4 \rho^2}{n}
    +
    \frac{\rho \conf}{n}
    }~.
\end{align*}
Now we specialize this result to a function $f\ntk_t - \tilde{f}\ntk_t$.
First we characterize the class radius $\rho$.
Using \Cref{prop:ftil-inf-bound}
\begin{align*}
  \|f\ntk_t - \tilde{f}\ntk_t\|_{\sH}
  &=
    \|(\bI - (\bI - \tfrac{2 \eta}{n} \, \bK)^{t}) (\by - \btily)\|_{\bK^{-1}}\\
  &\leq
    \|\bK^{-1}\|_{\mathrm{op}} \|\bI - (\bI - \tfrac{2 \eta}{n} \, \bK)^{t}\|_{\mathrm{op}} \|\by - \btily\|\\
  &\leq \frac{\sqrt{n}}{\lambda_0} \, A(R) = \rho
\end{align*}
by \Cref{prop:lip_approx} observing that $y_i - \tilde{y_i} = \fstar(\bx_i) - h(\bx_i)$.
Finally, to complete the bound we need a fixed design error, which is given by \Cref{lem:fntk-ftilntk-gap-sample}, namely $\|f\ntk_t - \tilde{f}\ntk_t\|_n \leq A(R)$.
Thus,
\begin{align*}
  \|f\ntk_t - \tilde{f}\ntk_t\|_2^2  
  &\leq A(R)^2
    + \Cr{sup-rad} \Bigg(
    \underbrace{
  A(R) \pr{ \frac{2 \pr{\frac{\sqrt{n}}{\lambda_0} \, A(R)}}{\sqrt{n}} + \sqrt{\frac{\pr{\frac{\sqrt{n}}{\lambda_0} \, A(R)} \conf}{n}} }
  +
  \frac{4 \pr{\frac{\sqrt{n}}{\lambda_0} \, A(R)}^2}{n}
  +
    \frac{\pr{\frac{\sqrt{n}}{\lambda_0} \, A(R)} \conf}{n}
    }_{(*)}
    \Bigg)
\end{align*}
Now,
\begin{align*}
  (*)
  &=
    \pr{\frac{2}{\lambda_0} + \frac{4}{\lambda_0^2}} \, A(R)^2 + A(R) \sqrt{\conf \, \frac{A(R)}{\lambda_0 \sqrt{n}}}
    +
    \conf \, \frac{A(R)}{\lambda_0 \sqrt{n}}\\
  &=
    \pr{2 \lambda_0 + 4} \, \frac{A(R)^2}{\lambda_0^2} + \lambda_0 \, \frac{A(R)}{\lambda_0} \sqrt{\conf \, \frac{A(R)}{\lambda_0 \sqrt{n}}}
    +
    \conf \, \frac{A(R)}{\lambda_0 \sqrt{n}}\\
  &\leq
    \pr{\pr{2 \lambda_0 + 4} \vee \lambda_0 \sqrt{\conf} \vee \conf \vee 1}
    \pr{
    \frac{A(R)^2}{\lambda_0^2} + \frac{A(R)}{\lambda_0} \sqrt{\frac{A(R)}{\lambda_0 \sqrt{n}}}
    +
    \frac{A(R)}{\lambda_0 \sqrt{n}}
    }
\end{align*}
Now some algebra leads to 
\begin{align*}
  x^2 + \frac{x^{\frac32}}{\sqrt[4]{n}} + \frac{x}{\sqrt{n}}
  \leq
  2 \pr{x^2 + \frac{1}{\sqrt{n}}} \qquad (x \geq 0)
\end{align*}
and using this inequality with $x = \frac{A(R)^2}{\lambda_0^2}$ gives us
\begin{align*}
  (*) \leq
  2 \pr{\pr{2 \lambda_0 + 4} \vee \lambda_0 \sqrt{\conf} \vee \conf \vee 1} \pr{\frac{A(R)^2}{\lambda_0^2} + \frac{1}{\sqrt{n}}}
\end{align*}
and so
\begin{align*}
  \|f\ntk_t - \tilde{f}\ntk_t\|_2^2  
  &\leq A(R)^2
  + 2 \Cr{sup-rad} \pr{\pr{2 \lambda_0 + 4} \vee \lambda_0 \sqrt{\conf} \vee \conf \vee 1} \pr{\frac{A(R)^2}{\lambda_0^2} + \frac{1}{\sqrt{n}}}~.
\end{align*}
The proof of \Cref{lem:gap-fk-ftilk} is now complete.
\end{proof} %

\subsubsection{Proof of \Cref{thm:random-design}}
\label{sec:proof-thm-random-design}
  For any $t \in \mathbb{N}$, we have a decomposition
  \begin{align*}
    \frac14 \|f_t - \fstar\|_2^2
    &\leq
      \|f_t - f\ntk_t\|_2^2
    +
      \|f\ntk_t - \tilde{f}\ntk_t\|_2^2
    +
      \|\tilde{f}\ntk_t - h\|_2^2
    +
      \|h - \fstar\|_2^2 \nonumber\\
    &\leq
      \frac{\poly_{4}\pr{B_y^2 \,\frac{n}{\lambda_0}, \conf}}{\sqrt{m}}
      +
      \pr{1 + \frac{\Cr{gap}}{\lambda_0^2}} A(R)^2 + \frac{\Cr{gap}}{\sqrt{n}}
      +
      \|\tilde{f}\ntk_t - h\|_2^2      
      +
      A(R)^2      
  \end{align*}
  which holds with probability at least $1-2 (1 + n) e^{-\conf} - e^{-\conf}, \conf \geq 1$, over $(\btheta_0, \bx_1, \ldots, \bx_n)$.
  \begin{itemize}
  \item Term $\|f_t - f\ntk_t\|_2^2$
    is a cost of approximating prediction of a shallow neural network by a \ac{KLS} predictor, and which is bounded by \Cref{thm:f-coupling} with high probability.
    \item Term $\|f\ntk_t - \tilde{f}\ntk_t\|_2^2$ is a gap between the \ac{KLS} predictor and a \ac{KLS} predictor learning the approximator $h$.
      The term is bounded in \Cref{lem:gap-fk-ftilk} with high probability.
  \item Term $\|h - \fstar\|_2^2$ is an approximation error of a Lipschitz function by a function $h$ \ac{RKHS} (an approximator), and which is controlled by \Cref{prop:lip_approx}.  
  \end{itemize}
By \Cref{ass:excess-kernel} we have $\|\tilde{f}\ntk_{\hT} - h\|_2^2 \leq (1 + R) \, \kexcess$ and so we set $t = \hT$ in the above.
\paragraph{Tuning of $R$.}
We observe that we need to control the trade-off between $R$ and $A(R)^2$, indicated below by a function $\sC(R)$:
\begin{align*}
  \frac14 \|f_{\hT} - \fstar\|_2^2
  &\leq
  \frac{\poly_{4}\pr{B_y^2 \,\frac{n}{\lambda_0}, \conf}}{\sqrt{m}}
  +
    \underbrace{
      \pr{2 + \frac{\Cr{gap}}{\lambda_0^2}} A(R)^2
      +
      R \, \kexcess
    }_{\sC(R) =}
      +
    \kexcess
    + \frac{\Cr{gap}}{\sqrt{n}}~.
\end{align*}
The trade-off in $R$ is controlled using \Cref{prop:tradeoff} with $x = \pr{2 + \frac{\Cr{gap}}{\lambda_0^2}}$ and $y = \kexcess$,
that is
\begin{align*}
  \min_{R \geq (\Cr{lip-approx} \vee \Lambda^2)} \sC(R)
  \leq
    \polylog(\tfrac{1}{\kexcess}, \tfrac{1}{\lambda_0}, \Cr{gap}) \,    
    \pr{2 + \frac{\Cr{gap}}{\lambda_0^2}}^{1-\frac{2}{d}} \, \Lambda^2 \, \pr{\kexcess}^{\frac{2}{d}}
\end{align*}
where polylogarithmic term in the above is
\begin{align*}
  \polylog\pr{\frac{1}{\kexcess}, \frac{1}{\lambda_0}, \Cr{gap}} = 1 + \Cr{lip-approx}^2 \, \ln_+^2\pr{\pr{\pr{2 + \frac{\Cr{gap}}{\lambda_0^2}} / \kexcess}^{\frac12-\frac{1}{d}}}~.
\end{align*}
Note that for \Cref{prop:tradeoff} to hold we also require the following to hold:
\begin{align}
  \label{eq:random-design-tech-assumption}
  \pr{\frac{\kexcess}{2 + \Cr{gap}/\lambda_0^2}}^{\frac{2}{d}-1} \geq ((\Cr{lip-approx} / \Lambda^2) \vee 1)
\end{align}
which holds by \Cref{ass:excess-kernel-decreasing}.
Proof of \Cref{thm:random-design} is now complete. %
\subsection{Excess risk under RWY stopping rule}
\label{sec:proofs-rwy}
\subsubsection{Lemmata}
\label{sec:lemmata-rwy}
An essential part of the proof of \Cref{thm:random-design-rwy} relies on establishing that early-stopped \ac{GD} learns on \ac{RKHS} at an optimal rate.
For this purpose we will employ a localized Rademacher complexity bound, \Cref{thm:rad} as in \citep{raskutti2014early}.
Most of the fact summarized in this sub-section are from \citep{raskutti2014early,wainwright2019high}.

The localized population Rademacher complexity of a functions class $\sF$ is defined as
\begin{align*}
  \mathrm{Rad}(\sqrt{r}, \sF) = \E \sup_{f \in \sF, \, \|f\|_2 \leq \sqrt{r}} \abs{\frac1n \sum_{i=1}^n \gamma_i f(\bx_i)} \qquad (r \geq 0)
\end{align*}
where $(\bx_1, \ldots, \bx_n) \sim P_X^n$ are i.i.d. and
where $(\gamma_1, \ldots, \gamma_n) \sim \unif\{\pm 1\}^n$ are i.i.d.
Then we have the following relationship between $\ell^2$ and empirical function norms, which holds uniformely over the class and which is controlled by the localized complexity.
\begin{theorem}[{\citet[Theorem 14.1]{wainwright2019high}}]
  \label{thm:rad}
  Given a star-shaped and $b$-uniformly bounded function class $\sF$, let $r^{\star}$ be any positive solution of the inequality
  \begin{align*}
    \mathrm{Rad}(\sqrt{r}, \sF) \leq \frac{r}{b}~.
  \end{align*}
  Then, for any $x \geq \sqrt{r^{\star}}$, with probability at least $1-\Cr{rad-1} e^{-\Cr{rad-2} \frac{n x^2}{b^2}}$ we have
  \begin{align*}
    \Big| \|f\|_2^2 - \|f\|_n^2 \Big| \leq \frac12 \, \|f\|_2^2 + \frac{x^2}{2} \qquad \text{for all} \, f \in \sF~,
  \end{align*}
  where $\Cl{rad-1}, \Cl{rad-2}$ are universal constants.
\end{theorem}
Now, similarly to the empirical complexity given in \Cref{def:rad}, we require the population complexity, which depends on the spectrum of a kernel function $(\mu_i)_{i=1}^{\infty}$,
\begin{align*}
  \sR(x) = \sqrt{\frac{1}{n} \sum_{i=1}^{\infty} \min\cbr{x^2, \mu_i}} \qquad (x > 0)~.
\end{align*}
We will consider classes which are subsets of \ac{RKHS} and so the following lemma provides the first step of controlling the localized complexity through the $\sR$-complexity of \ac{RKHS}.
\begin{lemma}[{\citet[Corollary 14.5]{wainwright2019high}}]
  \label{lem:rad-ker}
  Let $\sF = \{f \in \sH \mid \|f\|_{\sH} \leq 1\}$ be the unit ball of an \ac{RKHS} with eigenvalues $(\mu_i)_{i=1}^{\infty}$.  
  Then, for $\sR(\cdot)$ given in \Cref{def:rad}, the localized population Rademacher complexity satisfies
  \begin{align*}
    \mathrm{Rad}(\sqrt{r}, \sF) \leq \sqrt{2} \, \sR(\sqrt{r})~.
  \end{align*}
\end{lemma}

Moreover, the following fact will be used for controlling the positive solution to $\sR(\sqrt{r}) \leq r$, thus providing means for controlling the condition of \Cref{thm:rad}.
\begin{proposition}[{Within the proof of \citep[Corollary 3]{raskutti2014early}}]
  \label{prop:critical-spectrum}
  Fix $b > 0$ and let
  $r^{\star}_b$ be the smallest positive solution to $\sR(\sqrt{r^{\star}_b}) \leq r^{\star}_b / b$.
  Assume that the \ac{RKHS} has a polynomial eigenvalue decay rate $\mu_k \leq \Cr{rkhs-rate} k^{-\beta}$ for some $\beta > 1$ and some constant $\Cr{rkhs-rate}$.  
  Then, there exists a constant $\Cr{r-decay}$ which depends only on $\Cr{rkhs-rate}$ and $\beta$, such that
  \begin{align*}
    r^{\star}_b = \Cr{r-decay} \, (b^2)^{\frac{\beta}{\beta + 1}} n^{-\frac{\beta}{\beta + 1}}~.
  \end{align*}
\end{proposition}
Finally, the following handy lemma establishes the connection between the empirical critical radius (which is essentially arises from \Cref{ass:tuning}) and the population one.
\begin{lemma}[{\citet[Lemma 11]{raskutti2014early}, \citet[Prop. 14.25]{wainwright2019high}}]
  \label{lem:emp-r-to-pop-r}
  Suppose that $\hat{r}$ and $r^{\star}$ satisfy
  \begin{align*}
    \hat{\sR}(\sqrt{\hat{r}}) \leq \frac{\hat{r}}{2 e \sigma}~, \qquad
    \sR(\sqrt{r^{\star}}) \leq \frac{r^{\star}}{40 \sigma}~.
  \end{align*}
  Then, there exist constants $\Cr{radii-1}$ and $\Cr{radii-2}$ such that $\P(\tfrac14 \, r^{\star} \leq \hat{r} \leq r^{\star}) \geq 1 - \Cr{radii-1} e^{-\Cr{radii-2} n r^{\star}}$.
\end{lemma}
Notably we have the following simple corollary of \Cref{prop:critical-spectrum} and \Cref{lem:emp-r-to-pop-r}.
\paragraph{\Cref{cor:emp-r} (restated).}\emph{
  Under conditions of \Cref{prop:critical-spectrum} and \Cref{lem:emp-r-to-pop-r},
  with probability at least  
  $1 - \Cr{radii-1} \exp\pr{-\Cr{radii-2} n \, \Cr{r-decay} (40^2 \sigma^2)^{\frac{\beta}{\beta + 1}} n^{-\frac{\beta}{\beta + 1}}}$
  we have
  \begin{align*}
    \hat{r} \leq \Cr{r-decay} (40^2 \sigma^2)^{\frac{\beta}{\beta + 1}} n^{-\frac{\beta}{\beta + 1}}~.
  \end{align*}
}
\begin{lemma}
  \label{lem:ftil-ker-norm}
  Let $\tilde{f}\ntk_t$ be a \ac{GD}-trained {KLS} predictor (\Cref{def:ntrf-ntk-predictors}) given training sample $(\bx_i, \tilde{y}_i)_{i=1}^n$,
  where $\tilde{y}_i = h(\bx_i) + \ve_i$.
  Consider the stopping time $\hT$ given by the rule of \cref{eq:stopping} and let the critical radius $\hat{r}$ be as in \Cref{def:rad}.
  Then, there exists a universal constant $\Cl{rwy-lemma9}$, such that for any $t \leq \hT$, the following facts hold:
  \begin{align*}
    \P_{\bve}\pr{\|\tilde{f}\ntk_t\|_{\sH}^2 \leq 3 + 2 \|h\|_{\sH}^2 } \geq 1
    - e^{-\Cr{rwy-lemma9} n \hat{r}}~.
  \end{align*}
\end{lemma}
\begin{proof}
Throughout the proof, for any $t \in \mathbb{N}$, denote $\bM_t = \bK^{-\frac12} \pr{\bI - (\bI - \tfrac{2 \eta}{n} \bK)^t}$.  
By \Cref{prop:ftil-inf-bound}, %
\begin{align*}
  \|\tilde{f}\ntk_t\|_{\sH}^2
  &= \|\btily\|_{\bM_t}^2\\
  &= \|\bh_n\|_{\bM_t}^2 + 2 \ip{\bh_n, \bve_n}_{\bM_t} + \|\bve\|_{\bM_t}^2
\end{align*}
where we have identity by definition of $\btily$, and where $\bh_n = (h(\bx_1), \ldots, h(\bx_n))$.
Now, to bound the above we use:
\begin{proposition}[Extract from the proof of {\citep[Lemma 9]{raskutti2014early}}]
  For $g \in \sH$, $\|g\|_{\sH} \leq 1$, let $\bg_n = (g(\bx_1), \ldots, g(\bx_n))$.
  Then, there exists a universal constant $\Cr{rwy-lemma9}$, such that for any $t \leq \hT$, the following facts hold:
  \begin{align*}
    &\|\bg_n\|^2_{\bM_t} \leq 1~,\\
      &\P_{\bve}(\|\bve\|_{\bM_t}^2 \leq 2 \quad \mathrm{and} \quad |\ip{\bve, \bg_n}_{\bM_t}| \leq 1) \geq 1 - e^{-\Cr{rwy-lemma9} n \hat{r}}~.
  \end{align*}  
\end{proposition}
Apply the above with $g = h/\|h\|_{\sH}$.
Then, we get $\|\bh_n\|_{\bM_t}^2 \leq \|h\|_{\sH}^2$.
Moreover, with high probability, we have $\|\bve\|_{\bM_t}^2 \leq 2$ and $|\ip{\bve, \bh_n}_{\bM_t}| \leq \|h\|_{\sH}$.
So,
\begin{align*}
  \|\tilde{f}\ntk_t\|_{\sH}^2 \leq \|h\|_{\sH}^2 + 2 \|h\|_{\sH} + 2 = 1 + (\|h\|_{\sH} + 1)^2 \leq 3 + 2 \|h\|_{\sH}^2~.
\end{align*}
This completes the proof of \Cref{lem:ftil-ker-norm}.
\end{proof}
\subsubsection{Proof of \Cref{thm:random-design-rwy}}
We will use \Cref{thm:rad} to give a high-probability bound on $\|\tilde{f}\ntk_{\hT} - h\|_2^2$, but before have to establish that functions under consideration are bounded.
  Triangle and Cauchy-Schwartz inequalities in combination with the fact that $\sup_{\bx \in \Xdom} \kappa(\bx, \bx) \leq 1$ and \Cref{lem:ftil-ker-norm} give us that for any fixed $t \leq \hT$,
  \begin{align*}
    \|\tilde{f}\ntk_t - h\|_{\infty}
    &\leq
      \|\tilde{f}\ntk_t\|_{\infty} + \|h\|_{\infty}\\
    &\leq
      \|\tilde{f}\ntk_t\|_{\sH} + \|h\|_{\sH}\\
    &\leq
      B(R) \qquad \text{where} \qquad B(R) = \sqrt{3 + 2 R} + \sqrt{R}~,
  \end{align*}
  with probability at least $1-e^{-\Cr{rwy-lemma9} n \hat{r}}$ and where we assumed that $\|h\|_{\sH}^2 \leq R$.
  Now, abbreviating the event $\sE_{\mathrm{b}} = \{\|\tilde{f}\ntk_t - h\|_{\infty} \leq B(R)\}$, for any $x > 0$,
  \begin{align*}
    \P\pr{\|\tilde{f}\ntk_t - h\|_2^2 \geq x}
    &=
      \P\pr{\|\tilde{f}\ntk_t - h\|_2^2 \geq x \mid \sE_{\mathrm{b}}} \P(\sE_{\mathrm{b}})
      +
      \P\pr{\|\tilde{f}\ntk_t - h\|_2^2 \geq x \mid \neg \sE_{\mathrm{b}}} \P(\neg \sE_{\mathrm{b}})\\
    &\leq
      \P\pr{\|\tilde{f}\ntk_t - h\|_2^2 \geq x \mid \sE_{\mathrm{b}}}
      +
      e^{-\Cr{rwy-lemma9} n \hat{r}}~.
  \end{align*}
  To this end, to control $\P(\|\tilde{f}\ntk_t - h\|_2^2 \geq x \mid \sE_{\mathrm{b}})$, and we apply \Cref{thm:rad} to the function class  
  \begin{align*}
    \sF = \cbr{\bx \mapsto \frac{|f(\bx) - h(\bx)|}{B(R)} \bmid f,h \in \sH~, \quad \|f\|_{\sH}^2 \leq 3 + 2 R, \|h\|_{\sH}^2 \leq R}~.
  \end{align*}
  and so we have that $\sF$ is uniformely bounded by $1$.
  It remains to find a positive solution to the inequality $\mathrm{Rad}(\sqrt{r}, \sF) \leq r$ in terms of $r$.
  Using \Cref{lem:rad-ker},
  \begin{align*}
   \mathrm{Rad}\pr{\sqrt{r}, \sF} \leq \sqrt{2} \, \sR\pr{\sqrt{r}} \leq r \qquad (r > 0)
  \end{align*}
  and so letting $r^{\star}$ to be the smallest positive solution to the second inequality in the above we can apply \Cref{thm:rad} with $x^2 = r^{\star}$.
  Thus, under event $\sE_r$,
\begin{align*}
  \|\tilde{f}\ntk_{\hT} - h\|_2^2
  &\leq
    2 \|\tilde{f}\ntk_{\hT} - h\|_n^2
    +
    B(R)^2 \, r^{\star} \tag{W.p. $\geq 1-\Cr{rad-1} e^{-\Cr{rad-2} n r^{\star}}$}\\
  &\stackrel{(a)}{\leq}
    4 (R + 5) \, \hat{r}
    +
    B(R)^2 \, r^{\star} \tag{W.p. $\geq 1 - e^{-\Cr{rwy-var} n \hat{r}}$}\\ %
  &\stackrel{(b)}{\leq}
    4 (R + 5) \, \hat{r}
    +    
    6 (1 + R) \, r^{\star}\\
  &\leq
    9 (\hat{r} + r^{\star})
    +
    6 R (\hat{r} + r^{\star})
\end{align*}
where $(a)$ comes by controlling the empirical norm using \Cref{lem:rwy-fixed-design}, and $(b)$ comes by expanding $B(R)$ and using basic inequality $(x+y)^2 \leq 2x^2 + 2y^2$.

\paragraph{Bouding $r^{\star} + \hat{r}$}
Now, it remains to bound critical radii $\hat{r} + r^{\star}$.
Using \Cref{cor:emp-r} and \Cref{prop:critical-spectrum} with $b=1$ we have
\begin{align*}
  \hat{r} + r^{\star}
  &\leq
    \Cr{r-decay} \pr{(40^2 \sigma^2)^{\frac{\beta}{\beta + 1}} + 1} n^{-\frac{\beta}{\beta + 1}} \tag{For eigenvalue decay rate $\mu_k \leq \Cr{rkhs-rate} k^{-\beta}$}\\
  &=
    \underbrace{\Cr{r-decay} \pr{(40^2 \sigma^2)^{\frac{d}{d + 2}} + 1}}_{\Cl{random-design-radii}=} n^{-\frac{d}{d + 2}}
\end{align*}
where to get the last inequality we used the fact that for the \ac{RKHS} we consider, $\mu_k \leq \Cr{rkhs-rate} k^{-\frac{d}{2}}$ by \Cref{prop:kernel}.
Thus, plugging the bound on $\hat{r} + r^{\star}$ into the above
\begin{align*}
  \|\tilde{f}\ntk_{\hT} - h\|_2^2 \leq
  (1 + R) \,
  9 \Cr{r-decay} \pr{(40^2 \sigma^2)^{\frac{d}{d + 2}} + 1} \, n^{-\frac{d}{d + 2}}
\end{align*}

Now we determine the failure probability of the above by using a union bound over all the above high-probability bounds.
  
\bibliographystyle{plainnat}
\bibliography{learning}

\clearpage
\appendix
\clearpage
\section{Minor proofs}
\label{app:proofs}
\paragraph{\Cref{lem:rwy-fixed-design}~(restated, {\citet[Theorem 1]{raskutti2014early}})}
\emph{
  Consider the stopping time $\hT$ given by the rule of \cref{eq:stopping} and let the critical radius $\hat{r}$ be as in \Cref{def:rad}.
  Then,  there exists a universal constant $\Cr{rwy-var}$ such that,
  \begin{align*}
    \P_{\bve}\pr{\|\tilde{f}\ntk_{\hT} - h\|_n^2 \leq 2 (\|h\|_{\sH}^2 + 5) \hat{r}} \geq 1 - e^{-\Cr{rwy-var} n \hat{r}}~.
  \end{align*}
}
\begin{proof}
  \citep[Lemma 6]{raskutti2014early} shows that
  \begin{align*}
    \|\tilde{f}\ntk_t - h\|_n^2 \leq B_t^2 + V_t~, \qquad (t \in \mathbb{N})
  \end{align*}
  were $B_t^2$ is a squared bias and $V_t$ is a variance of an estimator, controlled by the following lemma:
  \begin{lemma}[{\citep[Lemma 7]{raskutti2014early}}]
    \label{lem:rwy-bias-variance}
      For all $t \geq 1$,
      \begin{align*}
        B_t^2 \leq \frac{1}{e \eta t} \, \|h\|_{\sH}^2~.
      \end{align*}
      Moreover, there is a universal constant $\Cr{rwy-var}$ such that, for any  $t \leq \hT$, with probability at least $1 - e^{-\Cr{rwy-var} n \hat{r}}$,
      \begin{align*}
        V_t \leq 5 \sigma^2 (\eta t) \hat{\sR}^2(1/\sqrt{\eta t})~.
      \end{align*}
    \end{lemma}
    The lemma presented here is almost identical as in \citep{raskutti2014early} with the difference that the bound on the squared bias has a factor bias $\|h\|_{\sH}^2$, which can be recovered by inspecting their proof (their work assumes $\|h\|_{\sH}^2 \leq 1$).
    To complete the proof we also need the following fact which follows from tuning of the stopping time:
    \begin{proposition}
      \label{prop:flow-radius-relationship}
      Consider the stopping rule of \cref{eq:stopping}.
      Then, we have $(\eta \hT)^{-1} \leq 2 \hat{r}$.
    \end{proposition}
    \begin{proof}[Proof of \Cref{prop:flow-radius-relationship}]
      From definition of the stopping rule and the empirical critical radius $\hat{r}$ (\Cref{def:rad}), we have $(\eta \hT)^{-1} \leq 2 (\eta (1 + \hT))^{-1} \leq 2 \hat{r}$.
      See also \citep[p. 352]{raskutti2014early}.
    \end{proof}
    Now, following the same reasoning as in \citep[p. 352]{raskutti2014early}, with high probability, for any $t \in [\hT]$,
  \begin{align*}
    \|\tilde{f}\ntk_t - h\|_n^2
    &\leq
    \frac{1}{e \eta t} \, \|h\|_{\sH}^2
    +
    5 \sigma^2 (\eta t) \hat{\sR}^2\pr{1/\sqrt{\eta t}}                      \tag{By \Cref{lem:rwy-bias-variance}}\\
    &\leq \frac{1}{e \eta t} \, \|h\|_{\sH}^2 + \frac{5}{\eta t}      \tag{By \Cref{def:rad}}
  \end{align*}
  Now, using the above in combination with \Cref{prop:flow-radius-relationship}
  \begin{align*}
    &\|\tilde{f}\ntk_{\hT} - h\|_n^2 \leq 2 \hat{r} \, \|h\|_{\sH}^2 + 10 \hat{r}~,
  \end{align*}
  which completes the proof.
\end{proof}
\paragraph{\Cref{prop:tradeoff} (restated).} \emph{
  For $A(R), \Lambda, \Cr{lip-approx}$ defined in \Cref{prop:lip_approx},
  and any $x,y > 0$,
  assume that relationship $\Lambda^2 \, \pr{\frac{z}{x}}^{\frac{2}{d}-1} \geq (\Cr{lip-approx} \vee \Lambda^2)$ holds.
  Then,
  \begin{align*}
    \min_{R \geq (\Cr{lip-approx} \vee \Lambda^2)} \cbr{x \, A(R)^2 + y \, R}
    \leq
  \pr{1 + \Cr{lip-approx}^2 \, \ln_+^2\Big(\pr{\frac{z}{x}}^{\frac{1}{d}-\frac12}\Big) } \Lambda^2 \, x^{1-\frac2d} \, y^{\frac2d}~.
  \end{align*}
}
\begin{proof}
  Consider function
  \begin{align*}
    R &\mapsto x \, A(R)^2
        + y \, R\\
      &=
        x \, \Cr{lip-approx}^2 \Lambda^2 \pr{\frac{R}{\Lambda^2}}^{-\frac{2}{d-2}} \ln^2\pr{\frac{\sqrt{R}}{\Lambda}}
        + y \, R
  \end{align*}
  by expanding definition of $A(R)$.
  Approximately minimizing the above in $R$  we obtain $R^{\star}=\Lambda^2 \, \pr{\frac{y}{x}}^{\frac{2}{d}-1}$.
  and assume that $R^{\star} \geq (\Cr{lip-approx} \vee \Lambda^2)$.
  Thus,
  \begin{align*}
    x \, A(R^{\star})^2
    &=
      x \, \Cr{lip-approx}^2 \Lambda^2 \pr{\pr{\frac{y}{x}}^{\frac{2}{d}-1}}^{-\frac{2}{d-2}} \ln^2\pr{\pr{\frac{y}{x}}^{\frac{1}{d}-\frac12}}\\
    &=
      \Cr{lip-approx}^2 \, \ln_+^2\Big(\pr{\frac{y}{x}}^{\frac{1}{d}-\frac12}\Big) \, \Lambda^2 \, x^{1-\frac2d} \, y^{\frac2d}
  \end{align*}
  where $\ln_+(\cdot)$ arises since its argument cannot be smaller than $1$ by requirement on $R^{\star}$.
  Following similarly,
  \begin{align*}
    y \, R^{\star}
    &=
      \Lambda^2 x^{1-\frac{2}{d}} y^{\frac{2}{d}}~.
  \end{align*}
Putting all together completes the proof.
\end{proof}
\end{document}